\documentclass{article}
\usepackage{iclr2025_conference}

\usepackage{hyperref}
\usepackage{microtype}
\usepackage{graphicx}
\usepackage{booktabs}
\usepackage{amsmath}
\usepackage{placeins}
\usepackage{amssymb}
\usepackage{mathtools}
\usepackage{amsthm}
\usepackage{algorithm}
\usepackage{algpseudocode}
\usepackage{caption}
\usepackage{pgfplots}
\usepackage{libertine}
\usepackage{xcolor}
\usepackage[T1]{fontenc}
\usepackage{tikz}
\usetikzlibrary{arrows.meta,fit,backgrounds,decorations.pathreplacing}
\usepackage{etoc}

\definecolor{accent}{HTML}{24405E}
\colorlet{tocaccent}{accent}

\definecolor{c0}{HTML}{2C6E91}
\definecolor{c1}{HTML}{D98841}
\definecolor{c2}{HTML}{6B9E55}
\definecolor{ink}{HTML}{1F2A33}
\definecolor{mut}{HTML}{6B7680}
\colorlet{panelbg}{accent!3!white}
\colorlet{panelln}{accent!22!white}
\colorlet{maskc}{accent!24!white}
\definecolor{boost}{HTML}{B23A48}
\definecolor{nullc}{HTML}{B9C2C9}

\tikzset{
  panel/.style={rounded corners=3pt, draw=panelln, line width=0.7pt, fill=panelbg},
  tok/.style={rounded corners=2pt, minimum width=8mm, minimum height=8mm,
              font=\small\bfseries, text=white, inner sep=0pt},
  mask/.style={tok, fill=maskc, text=ink!70, draw=panelln, font=\small},
  ttl/.style={font=\bfseries\small, text=ink},
  sub/.style={font=\scriptsize, text=mut},
  ar/.style={-{Latex[length=2.2mm,width=2mm]}, line width=0.9pt, ink},
}

\newcommand{\na}{\textendash{}}

\newcommand{\qprompt}[1]{{\footnotesize\itshape\color{mut}Prompt.\ #1}\par
  \smallskip{\color{panelln}\rule{\linewidth}{0.5pt}}\par\smallskip}
\newcommand{\qrow}[3]{\noindent{\footnotesize\color{#1}\scshape #2.}\
  {\footnotesize\color{ink}#3}\par\smallskip}
\newcommand{\qcard}[1]{\par\smallskip\noindent
  \begin{tikzpicture}\node[panel, inner sep=8pt]{\begin{minipage}{\dimexpr\linewidth-18pt\relax}#1\vspace{-\smallskipamount}\end{minipage}};\end{tikzpicture}\par\smallskip}

\colorlet{tango}{c0}
\colorlet{baseline}{boost}
\colorlet{refline}{mut}
\definecolor{attackb}{HTML}{B85C1E}
\definecolor{attackc}{HTML}{7B5EA7}
\pgfplotsset{compat=1.16}

\newtheoremstyle{thmthm}{0pt}{0pt}{\itshape}{0pt}{\bfseries\color{c0}}{.}{.5em}{}
\newtheoremstyle{thmlem}{0pt}{0pt}{\itshape}{0pt}{\bfseries\color{c0}}{.}{.5em}{}
\newtheoremstyle{thmsup}{0pt}{0pt}{\itshape}{0pt}{\bfseries\color{c0}}{.}{.5em}{}
\newtheoremstyle{thmcor}{0pt}{0pt}{\itshape}{0pt}{\bfseries\color{c0}}{.}{.5em}{}
\theoremstyle{thmthm}
\newtheorem{theorem}{Theorem}
\theoremstyle{thmlem}
\newtheorem{lemma}{Lemma}
\theoremstyle{thmsup}
\newtheorem{proposition}{Proposition}
\theoremstyle{thmcor}
\newtheorem{corollary}{Corollary}

\usepackage[capitalize,noabbrev]{cleveref}

\usepackage[most]{tcolorbox}
\tcolorboxenvironment{theorem}{
  enhanced,
  colback=accent!4, colframe=accent!50,
  boxrule=0.5pt, arc=2.5pt,
  left=7pt, right=7pt, top=3pt, bottom=3pt, boxsep=0pt,
  before skip=6pt, after skip=6pt}

\title{TANGO: Watermarking Masked Diffusion \\ Language Models in Token Pairs}

\author{Kasra Arabi \And Nir Weinberger \And Micah Goldblum \And Niv Cohen}

\newcommand{\ci}[2]{$#1^{\pm #2}$}

\begin{document}
\etocsettocdepth.toc{none}

\maketitle

\begin{abstract}
Masked-diffusion language models fill in masked positions in parallel and
in no fixed order. Most practical text watermarks assume left-to-right
generation. They key each token to the tokens before it, and in a
diffusion model those tokens may still be masked. A fixed green list needs
no such context, but it favors the same tokens at every position, so these
tokens appear more often in watermarked text. An attacker who compares
token frequencies in watermarked and unwatermarked text can recover the
list and forge text that the provider's own detector accepts. We present
\textbf{TANGO}, a watermark for masked-diffusion language models that keys
each new token to a nearby token that is already unmasked. A secret key
splits the vocabulary into color classes, and TANGO biases the new token
toward a color determined by the key and the nearby token's color. The
watermark is therefore embedded in \emph{pairs} of tokens. Because the
favored color changes from position to position, token frequencies stay much closer to those of unwatermarked text than under a fixed green list. Detection needs only the text and
the key, and it does not assume any unmasking order. On two
masked-diffusion models, TANGO detects nearly all unedited watermarked
texts and most edited ones, and frequency attacks that forge the fixed
green list fail against it.
\end{abstract}

\section{Introduction}
\label{sec:intro}

Language models now write text that readers cannot tell apart from human
writing, and their providers are increasingly asked to identify which texts
their models produced. Text watermarking addresses this need. The provider
embeds a secret statistical signal during generation and later tests for
its presence. A prominent method is the context-hashed green list
\citep{kirchenbauer2023watermark}, which derives its favored token set
from a hash of the preceding tokens. In left-to-right generation, these
tokens are already available when the next token is sampled.
Masked-diffusion language models
\citep{austin2021d3pm,sahoo2024mdlm,nie2025llada,dream2025} instead fill
positions out of order through a \emph{denoising loop}. Under confidence-based decoding, the model unmasks the positions it is most confident about first, regardless of their order in the sequence. When a position is unmasked, the tokens before it may still be masked, so the context-hashed green list cannot be computed.

Other watermarks run on diffusion models but face different problems. Gumbel-max sampling \citep{aaronson2023watermark}, which we call the Gumbel rule, can be applied at every denoising step \citep{bagchi2025ddlm}. However, when we key it on the preceding token, as in its autoregressive form, detection is weak in our experiments. A confidence-ordered sampler unmasks first the positions whose token the model is most certain of. With little uncertainty left there, the Gumbel rule tends to pick that same token, which likely explains why it embeds little signal (\Cref{sec:experiments-headline}).
The fixed green list of \citet{zhao2023provable}, which we call the red--green list, splits the vocabulary once into green and red tokens and favors the green ones. It needs no context and runs unchanged on diffusion models. However, it favors the same tokens at every position, so an attacker can estimate the green list from token frequencies in watermarked text. With the green list, the attacker can forge text that the detector accepts \citep{jovanovic2024stealing} or remove the watermark \citep{zhang2024mip}. A provider that needs to prevent forgery therefore cannot use the red--green list.

We present \textbf{TANGO}, a watermark that is designed for the denoising
loop and does not rely on a left-to-right generation order. A secret key
colors the vocabulary into $q$ classes. TANGO pairs each position with the
token at a fixed offset, which we call the position's \emph{tap}. If the
tap has already been generated when a position is unmasked, TANGO biases
the logits toward the one class whose color satisfies a keyed modular
relation with the tap's color. We call this relation the \emph{checksum}.
The detector recolors a candidate text with the key and counts the
positions where the checksum holds. Because the classes have near-equal
size, the checksum holds at about $1/q$ of the positions in unwatermarked
text, and the detector flags text where it holds much more often.
Detection needs only the text and the key, not the prompt or the model.

The watermark signal is therefore \emph{second order}. It comes from
pairs of token colors, not from any single token. When every class carries equal probability mass, the biases at different positions also cancel out on average. If the favored class is uniform and independent across positions, then under an idealized sampler every token has the same expected frequency as without the watermark (\Cref{thm:unforgeable}). Ranking tokens by how much more often they
appear in watermarked text therefore reveals, in expectation, no difference between the color classes. Empirically, this attack recovers and forges the
red--green list but stays at chance against TANGO, and none of its
forgeries is accepted (\Cref{sec:experiments-stealth}). An attacker who
knows TANGO's design can instead count token pairs. At every forging bias and every number of texts we test, forgeries built from these pair counts succeed less often against TANGO than forgeries of the red--green list
succeed against it.

Embedding the signal in token pairs raises two concerns. First, editing
either token of a pair can break it. TANGO therefore assigns tokens to
classes by \emph{meaning}. It sorts tokens by a keyed projection of their embeddings and cuts the
sorted list into equal-sized classes. Tokens with similar meanings then
tend to share a class, so a synonym substitution can keep the checksum. Second, forcing a relation
between neighbors could hurt the quality of the generated text. TANGO
favors a whole class of tokens, so the model still chooses freely within
that class, and quality degrades gradually as the bias grows. Under
synonym substitution and word insertion, TANGO detects nearly as many
texts as the red--green list. TANGO is more sensitive to word deletion than the red--green list.
Deleting a word breaks every pair that spans it, whereas the red--green
list scores each token independently (\Cref{sec:experiments-headline}).

\paragraph{Contributions.}
\begin{itemize}
    \item We introduce TANGO, a coloring-based watermark for masked-diffusion
    language models, which couples the color of each token to the color of
    one already-unmasked neighbor (\Cref{sec:method}).

    \item We prove three properties of TANGO (\Cref{sec:theory}). The
    expected detection score grows as $\sqrt{T}$, where $T$ is the number of
    scored pairs. When every class carries equal probability mass and the favored class is
    uniform and independent across positions, the watermark leaves the
    expected frequency of every token unchanged under an idealized sampler, so expected token frequencies do not reveal the key. Replacing tokens at random with rate $\rho$ multiplies
    the expected detection score by $(1-\rho)^{r}$, where $r$ is the number
    of tokens a checksum reads, so checksums over fewer tokens are more
    robust to edits.

    \item On two masked-diffusion models, TANGO detects nearly all unedited
    watermarked texts and most edited ones. Its detection rate on unedited
    text and its perplexity are comparable to those of the red--green list.
    At similar perplexity near our defaults, TANGO detects far more texts than the Gumbel rule
    (\Cref{sec:experiments-headline} and \Cref{tab:e7}). TANGO is also harder to forge, since frequency attacks that forge the
    red--green list fail against it (\Cref{sec:experiments-stealth}).

\end{itemize}

\section{Preliminaries}
\label{sec:prelim}

\textbf{Decoding in masked-diffusion language models.} A masked-diffusion model
starts from a block of masks and, at each step, chooses which masked
positions to fill and what to write there. Let $\mathcal{V}$ be the
vocabulary, $V = |\mathcal{V}|$, and $\mathsf{m} \notin \mathcal{V}$ the
mask symbol. Given a prompt, the model fills a length-$L$ block over $S$
denoising steps, producing states $x^{(0)}, \dots, x^{(S)} \in
(\mathcal{V} \cup \{\mathsf{m}\})^{L}$ with $x^{(0)}$ fully masked and
$x^{(S)} = (t_1, \dots, t_L)$ the output. At each step $j$, one
bidirectional forward pass of the model (parameters $\theta$) yields
marginals $p_\theta(x_i = v \mid x^{(j-1)})$ for every masked position.
The step samples a candidate at each masked position and keeps the $K$
most confident positions, re-masking the rest.
In LLaDA \citep{nie2025llada}, and in Dream \citep{dream2025} with the \texttt{maskgit\_plus} ordering that we use,
the confidence of a candidate is the probability
the model assigns to the sampled token, and the step's schedule sets
$K$.\footnote{Semi-autoregressive variants decode blocks left to right but
remain order-free within a block.} Unmasking is irreversible, so the order
in which positions are filled is a permutation of $\{1, \dots, L\}$ that depends on the data.

\textbf{Requirements.} This generation setting motivates three
requirements. \textit{(R1) Prompt-free detection.} The detector should use
only the candidate text and the secret key. \textit{(R2) Order-agnostic
detection.} Detection should not need the unmasking order, which the final
text does not reveal. \textit{(R3) Frequency-hidden key.} The token
frequencies of watermarked text should not make the secret key easy to
infer.

\textbf{Attacker.} We assume the attacker knows the watermarking method and every design parameter, but not the secret key. It can collect watermarked and
unwatermarked texts, edit a watermarked text, and generate text with a
language model it controls. Its goal is either to forge, that is, to
produce text that the provider's detector accepts, or to remove the
watermark from a text by editing it.

\section{TANGO}
\label{sec:method}

\textbf{Overview.} Before generation, the secret key $k$ pseudorandomly
assigns every vocabulary token one of $q$ colors through a coloring $\chi :
\mathcal{V} \to \mathbb{Z}_q$ that is balanced (\Cref{sec:method-coloring}), where $\mathbb{Z}_q =
\{0, \dots, q-1\}$ with arithmetic modulo $q$. During generation, TANGO
considers each masked position $i$ whose tap, the token at position $i - \delta$
at a small fixed lag $\delta$, is already unmasked. It adds a bias $\beta > 0$
to the logits of every token of the one color that satisfies the checksum
with the tap's color. The step then selects which candidates to unmask by
the sampler's usual rule, keeping the most confident ones, with
confidences computed from the biased distribution. Detection needs neither
the model nor the prompt. It recolors the text with $\chi$, counts the
positions where the checksum holds, and compares the count with the
binomial distribution expected by chance (\Cref{fig:schematic}).

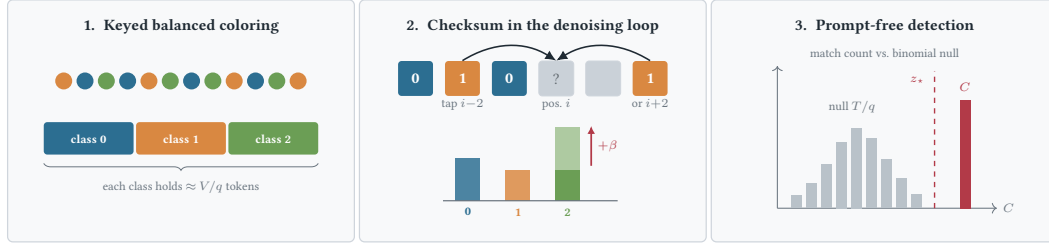
\begin{figure}[t]
\centering
\resizebox{\textwidth}{!}{\begin{tikzpicture}[font=\sffamily]

\begin{scope}[shift={(0,0)}]
  \node (f1a) at (-3.35,-2.25) {}; \node (f1b) at (3.35,2.35) {};
  \node[ttl] at (0,1.95) {1.\; Keyed balanced coloring};

  \def\xs{-2.35}
  \foreach \i/\col in {0/c1,1/c0,2/c2,3/c0,4/c1,5/c2,6/c0,7/c2,8/c1,9/c0,10/c2,11/c1}{
    \node[circle,fill=\col,minimum size=3.4mm,inner sep=0pt]
      at ({\xs+0.427*\i},0.85) {};
  }

  \node[fill=c0, rounded corners=2pt, minimum width=18mm, minimum height=6.5mm,
        text=white, font=\scriptsize\bfseries] at (-1.85,-0.30) {class 0};
  \node[fill=c1, rounded corners=2pt, minimum width=18mm, minimum height=6.5mm,
        text=white, font=\scriptsize\bfseries] at (0,-0.30) {class 1};
  \node[fill=c2, rounded corners=2pt, minimum width=18mm, minimum height=6.5mm,
        text=white, font=\scriptsize\bfseries] at (1.85,-0.30) {class 2};

  \draw[decorate,decoration={brace,amplitude=4pt,mirror},mut]
    (-2.75,-0.80) -- (2.75,-0.80);
  \node[sub,anchor=north] at (0,-1.02)
    {each class holds ${\approx}\,V/q$ tokens};
\end{scope}

\begin{scope}[shift={(7.05,0)}]
  \node (f2a) at (-3.35,-2.25) {}; \node (f2b) at (3.35,2.35) {};
  \node[ttl] at (0,1.95) {2.\; Checksum in the denoising loop};

  \def\y{1.10}
  \node[tok, fill=c0, minimum width=7mm, minimum height=7mm] (g0) at (-2.35,\y-0.2) {0};
  \node[tok, fill=c1, minimum width=7mm, minimum height=7mm] (g1) at (-1.41,\y-0.2) {1};
  \node[tok, fill=c0, minimum width=7mm, minimum height=7mm] (g2) at (-0.47,\y-0.2) {0};
  \node[mask, minimum width=7mm, minimum height=7mm]         (gi) at (0.47,\y-0.2)  {?};
  \node[mask, minimum width=7mm, minimum height=7mm]         (g4) at (1.41,\y-0.2)  {\,};
  \node[tok, fill=c1, minimum width=7mm, minimum height=7mm] (g5) at (2.35,\y-0.2)  {1};

  \draw[ar] (g1.north) to[bend left=32] (gi.north);
  \draw[ar] (g5.north) to[bend right=32] (gi.north);
  \node[sub,anchor=north,inner sep=1.5pt] at (gi.south) {pos.\ $i$};
  \node[sub,anchor=north,inner sep=1.5pt] at (g1.south) {tap $i{-}2$};
  \node[sub,anchor=north,inner sep=1.5pt] at (g5.south) {or $i{+}2$};

  \def\by{-1.55}
  \fill[c0!85] (-1.55,\by) rectangle ++(0.50,0.85);
  \fill[c1!85] (-0.55,\by) rectangle ++(0.50,0.62);
  \fill[c2]    (0.45,\by)  rectangle ++(0.50,0.62);
  \fill[c2!55] (0.45,{\by+0.62}) rectangle ++(0.50,0.85);
  \draw[boost, line width=1pt, -{Latex[length=1.6mm,width=1.6mm]}]
    (1.17,{\by+0.70}) -- node[right,font=\scriptsize,boost]{$+\beta$} (1.17,{\by+1.50});
  \draw[mut, line width=0.7pt] (-1.80,\by) -- (1.70,\by);
  \node[font=\scriptsize\bfseries, text=c0] at (-1.30,{\by-0.22}) {0};
  \node[font=\scriptsize\bfseries, text=c1] at (-0.30,{\by-0.22}) {1};
  \node[font=\scriptsize\bfseries, text=c2] at (0.70,{\by-0.22})  {2};
\end{scope}

\begin{scope}[shift={(14.1,0)}]
  \node (f3a) at (-3.35,-2.25) {}; \node (f3b) at (3.35,2.35) {};
  \node[ttl] at (0,1.95) {3.\; Prompt-free detection};
  \node[sub, anchor=north] at (0,1.62) {match count vs.\ binomial null};

  \begin{scope}[shift={(-2.15,-1.70)}]
    \draw[->, mut, line width=0.6pt] (0,0) -- (4.4,0) node[right,sub]{$C$};
    \draw[->, mut, line width=0.6pt] (0,0) -- (0,2.85);
    \foreach \x/\h in {0.28/0.25,0.58/0.49,0.88/0.88,1.18/1.33,1.48/1.61,1.78/1.40,2.08/0.98,2.38/0.60,2.68/0.28}{
      \fill[nullc] (\x,0) rectangle ++(0.22,\h);
    }
    \node[sub,nullc!50!ink] at (1.5,2.02) {null $T/q$};
    \fill[boost] (3.65,0) rectangle ++(0.22,2.17);
    \draw[boost, line width=0.8pt, dashed] (3.15,0) -- (3.15,2.75);
    \node[font=\scriptsize\bfseries, boost, anchor=south] at (3.76,2.22) {$C$};
    \node[sub, boost, anchor=east] at (3.07,2.57) {$z_\star$};
  \end{scope}
\end{scope}

\begin{scope}[on background layer]
  \node[panel, fit=(f1a)(f1b), inner sep=0pt] {};
  \node[panel, fit=(f2a)(f2b), inner sep=0pt] {};
  \node[panel, fit=(f3a)(f3b), inner sep=0pt] {};
\end{scope}

\end{tikzpicture}
}
\caption{\textbf{How TANGO works.} \textbf{(1)}~A secret key sorts the
vocabulary by a keyed embedding score and cuts it into $q$ color classes
of near-equal size, so that tokens with similar embeddings tend to share
a color. \textbf{(2)}~When the tap of a masked position (or its right neighbor at
$i+2$) is already unmasked, a logit bias $+\beta$ favors the one class that satisfies the checksum with that token's color. The favored class changes from
position to position, so under the conditions of \Cref{thm:unforgeable}, no token is favored on average.
\textbf{(3)}~Detection counts the matches $C$ among the $T$ scored
positions and flags the text when $C$ exceeds the $\mathrm{Binom}(T,
1/q)$ null by more than a threshold $z_\star$ standard deviations.}
\label{fig:schematic}
\end{figure}
\subsection{Balanced colorings}
\label{sec:method-coloring}

A coloring is \emph{balanced} if every class contains either $\lfloor
V/q \rfloor$ or $\lceil V/q \rceil$ tokens. Balance makes plausible the
uniform-color hypothesis under which an unwatermarked position satisfies
the checksum with probability $1/q$ (\Cref{prop:null}), and it is the
token-count counterpart of the balance of class probabilities that
\Cref{thm:unforgeable} assumes. How tokens are assigned to classes
determines TANGO's robustness to edits.

We consider three ways of assigning tokens to colors. A \textbf{hash}
coloring partitions a keyed pseudorandom permutation of the vocabulary
into $q$ equal blocks. It ignores the meaning of the token, so replacing a
token by a synonym randomizes its color, and the pair it belongs to no
longer satisfies the checksum more often than chance. A \textbf{semantic-cluster} coloring is a balanced
clustering of the token embeddings. Tokens with similar embeddings, which
often occur near each other in text, then share a color, so unwatermarked text satisfies the
checksum more often than chance. At small $q$ this raises the
false-positive rate far above its nominal value ($60\%$ at $q = 2$ and
$30.5\%$ at $q = 3$, \Cref{app:sweeps}). We propose
\textbf{semantic-quantile} coloring, which sorts tokens by the projection
$\langle \mathbf{e}_v / \|\mathbf{e}_v\|, \mathbf{r}_k \rangle$ of the
unit-normalized embedding of each token $v$ onto a key-seeded direction
$\mathbf{r}_k$, and cuts the sorted list into $q$ equal-count buckets.
Synonyms tend to have nearby projections, so this coloring is designed
to let a substitution keep the color more often than chance and thereby
preserve part of the signal under editing. Unlike the semantic-cluster
coloring, it keeps the false-positive rate close to the nominal value
(\Cref{app:sweeps}).

Balance holds exactly for token counts, but text does not use every
token equally often. In natural text, one class then carries more or less
than $1/q$ of the probability, the chance rate drifts away from $1/q$, and
token frequencies begin to reveal the key (\Cref{prop:leak}). ZCA
whitening \citep{kessy2018whitening} of the embeddings reduces this drift
but does not remove it (\Cref{sec:experiments-null}). Our defaults do not
whiten, and we set the detection threshold per model on unwatermarked
text.

\subsection{Generation: enforcing the checksum inside the denoising loop}
\label{sec:method-generation}

The red--green list \citep{zhao2023provable} favors one token set for the
entire generation. TANGO instead chooses the favored class at each
position from a token that is already unmasked. We say that position $i$
is \emph{eligible} at a denoising step if it is still masked and its tap
$i - \delta$ is already unmasked. A position never waits for its tap, so
generation proceeds in whatever order the sampler chooses. At an eligible
position, TANGO favors the tokens that satisfy the checksum. With $a_\delta \in \mathbb{Z}_q \setminus \{0\}$ the
coefficient of the tap and $b \in \mathbb{Z}_q$ the target residue,
the checksum is
\begin{equation}
\label{eq:constraint}
\big( \chi(t_i) + a_\delta\, \chi(t_{i-\delta}) \big) \bmod q \;=\; b .
\end{equation}
The new token's color enters with unit weight, so once the tap is
unmasked, exactly one class satisfies \Cref{eq:constraint},
\begin{equation}
\label{eq:satisfying}
s_i \;=\; \big( b - a_\delta\, \chi(t_{i-\delta}) \big) \bmod q .
\end{equation}
TANGO adds $\beta$ to the logits of every token in class $s_i$ (after
classifier-free guidance on LLaDA), and the step then samples and unmasks
candidates as usual (\Cref{alg:bcc} in \Cref{app:protocol}). The
general form uses a set $\mathcal{D}$ of tap lags and replaces the tap
term of \Cref{eq:constraint} by $\sum_{\delta \in \mathcal{D}} a_\delta\,
\chi(t_{i-\delta})$, so that a checksum reads $|\mathcal{D}| + 1$ tokens.
By default we use a single tap, $\mathcal{D} = \{2\}$
(\Cref{sec:experiments-sweeps}).

\textbf{When the tap is still masked.} The sampler often commits a
position before its tap. With enforcement from the left only
(\emph{one-sided} enforcement), this leaves about $40\%$ of scored pairs
on LLaDA unenforced (\Cref{tab:e9-screen}). With a single tap and
$\gcd(a_\delta, q) = 1$, \Cref{eq:constraint} can also be solved for the
tap, so a position whose right neighbor at $i + \delta$ is already
committed is biased toward the color that completes that pair. With this
\emph{either-side} enforcement, most pairs are enforced at whichever
endpoint the sampler commits second. Detection and the guarantees of
\Cref{sec:theory} are unchanged, because each biased token still
completes exactly one pair through the same checksum.

\subsection{Detection}
\label{sec:method-detection}

The detector uses only the candidate text and the key, so it meets the
prompt-free requirement (R1). Given a text $t_1, \dots, t_L$, position $i$ is scored whenever its
tap $t_{i-\delta}$ exists, that is, whenever $i > \delta$. Let
$\mathcal{T}$ be the set of scored positions and $T = |\mathcal{T}|$
their number. Each scored position $i$ indexes the pair $(t_{i-\delta},
t_i)$, so we also call $T$ the number of scored pairs. Both depend only
on the observed text, so the detector never needs the order in which
positions were unmasked (R2). At each scored
position the detector recolors the token and its tap with $\chi$ and
checks whether \Cref{eq:constraint} holds. The statistic $C$ counts the
positions where it does, and $z$ measures how far $C$ exceeds its chance
value $T/q$, in standard deviations of the binomial null,
\begin{equation}
\label{eq:detector-z}
C \;=\; \sum_{i \in \mathcal{T}} \mathbf{1}\!\left[
\big( \chi(t_i) + a_\delta\, \chi(t_{i-\delta}) \big) \bmod q = b
\right],
\qquad
z \;=\; \frac{C - T/q}{\sqrt{T}\,\sigma_0},
\qquad
\sigma_0^2 \;=\; \tfrac{1}{q}\bigl(1 - \tfrac{1}{q}\bigr),
\end{equation}
where $\mathbf{1}[\cdot]$ is the indicator and $\sigma_0^2$ is the
variance of one match under the null.
The text is declared watermarked when $z$ exceeds a threshold $z_\star$.
Choosing $z_\star$ so that unwatermarked text is flagged at a chosen
false-positive rate is what we call \emph{calibration}. Under three
approximations stated in \Cref{app:optimal} and with one-sided
enforcement, this count test is the likelihood-ratio test among texts
with the same $T$. At every false-positive rate that a threshold on $C$
attains, no detector then detects more (\Cref{prop:optimal}).

The binomial null is exact under two conditions. First, the scored
token's color is uniform over the $q$ classes and independent of the
tap's color. Adding the tap's term then gives a uniform residue, so the
checksum holds with probability $1/q$ at every position. Second, the
match events at different positions are independent. Then $C$ is exactly
binomial.

\begin{proposition}[Binomial null]
\label{prop:null}
Assume that at each scored position the color $\chi(t_i)$ is uniform on
$\mathbb{Z}_q$ and independent of the colors of its taps, and that the
match indicators are mutually independent across positions. Then $C \sim
\mathrm{Binom}(T, 1/q)$ for every residue $b$, and $z$ converges in
distribution to $\mathcal{N}(0,1)$ as $T \to \infty$. (Proof in
\Cref{app:proof-null}.)
\end{proposition}

Real text satisfies neither condition exactly. For a hash coloring,
uniformity holds in expectation over the key to within $1/V$
(\Cref{prop:keyconc}). For the semantic-quantile coloring, which is not
uniformly random, we measure the class masses directly (\Cref{tab:residue-diag}). Both kinds of independence fail because neighboring tokens are correlated, most strongly when adjacent. Under the default key, the
null mean score of natural text ranges from $-0.93$ to $+0.98$ across
residues at lag $1$ and from $-0.42$ to $+0.24$ at lag $2$
(\Cref{tab:residue-diag}). This correlation is why the default tap skips
the adjacent token.

\subsection{Theoretical guarantees}
\label{sec:theory}
\label{sec:theory-detect}
\label{sec:theory-unforge}
\label{sec:theory-robust}

We analyze how much text the detector needs (\Cref{prop:ez}), what token
frequencies reveal about the key (\Cref{thm:unforgeable}), and how edits
weaken the score (\Cref{thm:robust}).

\textbf{Setting and assumptions.}
Let $M_i \in \{0, 1\}$ indicate whether the checksum holds at scored
position $i$, so that $C = \sum_{i \in \mathcal{T}} M_i$. On the
generation side, let $p_i$ be the model's distribution $p_\theta(x_i =
\cdot \mid x^{(j-1)})$ at position $i$ after temperature scaling, and
$m_s(p) = \sum_{v : \chi(v) = s} p(v)$ the total probability that $p$
assigns to tokens of class $s$. Adding $\beta$ to the logits of class $s$ multiplies the
probability of each of its tokens by $e^{\beta}$ and renormalizes. We
call the result the \emph{tilt} of $p$ toward $s$,
\begin{equation}
\label{eq:tilt}
(\mathsf{T}_s p)(v) \;=\; \frac{p(v)\, e^{\beta \mathbf{1}[\chi(v) = s]}}{1 +
\gamma\, m_s(p)},
\qquad
\gamma \;=\; e^{\beta} - 1 .
\end{equation}
A token is \emph{enforced} when the sampler draws it from a
tilt, and a scored pair is enforced when the token of the pair committed second was tilted toward that pair. The
\emph{enforced fraction} $\kappa$ is the expected fraction of scored
pairs that are enforced. The results use four idealizations, stated in
full in \Cref{app:setting}.
\textbf{(S)}~Sampling: candidates are drawn independently given the
state and the favored classes, from $\mathsf{T}_{s_i} p_i$ at enforced
positions and from $p_i$ elsewhere. They are committed by the unwatermarked
rule, and a committed token is treated as a draw from its candidate
distribution. This idealizes the confidence selection, whose confidences come
from the biased logits.
\textbf{(U)}~Unpredictability: given the history of states and $\chi$,
the favored classes of the positions enforced at a step are independent
and uniform on $\mathbb{Z}_q$. For a real key, the favored class is a
function of the tap's color. (U) drops this link, which is what the
detector counts, so we use (U) only for token frequencies. The detection
score uses a local form (U$'$) that keeps the tilt toward the pair being
scored.
\textbf{(B)}~Balance: $m_s(p_i) = 1/q$ for every class $s$ and position
$i$. It holds in expectation over the key to within $1/V$ for a hash
coloring (\Cref{prop:keyconc}).
\textbf{(I)}~Independence: the match indicators $M_i$ are mutually
independent. It is the second hypothesis of \Cref{prop:null}.
\Cref{sec:experiments-null} measures how far unwatermarked text departs
from (B) and (I).

\textbf{How much text detection needs.}
The detector needs enough scored positions for the matches caused by the
watermark to stand out from the noise of the null distribution. At an
enforced position the bias raises the probability that the checksum holds
above the chance value $1/q$. We call the average of this excess over all
scored positions the \emph{coupling margin} $\varepsilon$. The margin
grows with $\beta$ and shrinks when $p_i$ is already peaked, because a
confident model has little probability to move into the favored class.
Under (S), (U$'$), and (B), for a single tap with $\gcd(a_\delta, q) = 1$, unenforced pairs match at the chance rate, so
$\varepsilon = \kappa\, \bar\varepsilon_{\mathrm{enf}}$ with
$\bar\varepsilon_{\mathrm{enf}}$ the margin at enforced pairs
(\Cref{lem:factor}), and either-side enforcement raises $\varepsilon$ by
raising $\kappa$. The expected score is the ratio of the $T\varepsilon$
matches beyond chance to the null standard deviation
$\sqrt{T}\,\sigma_0$.

\begin{proposition}[Expected detection score]
\label{prop:ez}
For a fixed number $T$ of scored positions, let $\mu_i = \Pr(M_i = 1)$
under the watermarked process and $\varepsilon = \frac{1}{T}\sum_{i \in
\mathcal{T}}(\mu_i - 1/q)$. Then $\mathbb{E}[z] = \sqrt{T}\,\varepsilon /
\sigma_0$.
\end{proposition}

The identity follows from linearity of expectation and needs no
independence (\Cref{app:proof-detect}). The score thus grows like the
square root of the text length. With a fixed positive margin it reaches
$z_\star$ once $T \geq z_\star^2 \sigma_0^2 / \varepsilon^2$
(\Cref{cor:samples}, compared with LLaDA text in \Cref{fig:ztheory}), and
under (I) the false-negative rate decays exponentially in $T$
(\Cref{cor:detect}).

\textbf{First-order unforgeability.}
An attacker who reconstructs $\chi$ can forge, and requirement (R3) asks
that token frequencies not make $\chi$ easy to infer. The attacker can
rank tokens by their frequency shift $\Delta f(v) = f_{\mathrm{wm}}(v) -
f_{\mathrm{base}}(v)$ between watermarked and unwatermarked generations,
and the area under the ROC curve (AUC) of this ranking measures how well
it separates the color classes. Under balance, every tilt has the same
normalizer $1 + \gamma/q$ and each token is boosted by exactly one of the
$q$ tilts, so the tilts average to the unwatermarked distribution,
$\tfrac{1}{q} \sum_{s} \mathsf{T}_s p = p$ (\Cref{lem:mixture}).

\begin{theorem}[First-order unforgeability]
\label{thm:unforgeable}
Fix a coloring $\chi$. Under (S), (U), and (B), the generated text has the
same distribution with and without the watermark. Hence $\mathbb{E}[\Delta
f(v)] = 0$ for every $v \in \mathcal{V}$, and any attacker who ranks
each token $v$ by $g(\mathbb{E}[\Delta f(v)])$, for a fixed function $g$,
separates the color classes no better than chance (AUC $\tfrac{1}{2}$).
(Proof in
\Cref{app:proof-unforge}.)
\end{theorem}

In words, when the favored class behaves like a fresh uniform draw at
every position, the watermark pushes no token's expected frequency in a
direction set by its color. The theorem does not say that one key
produces unwatermarked text, because (U) drops the link that the detector
counts. It also covers only expected shifts. The attacker of
\Cref{sec:experiments-stealth} ranks tokens by shifts measured on finitely
many texts, so its chance-level result is an empirical finding. When (B)
fails, the cancellation is inexact. The leak is proportional to the class
imbalance and, for small $\gamma$, second order in $\gamma$
(\Cref{prop:leak}). The red--green list favors the same tokens at every
position. Its favored class never varies, so nothing averages out, and
ranking by frequency shift recovers the green list
(\Cref{prop:greensep}).

\textbf{Edit robustness.}
An edit can break a checksum whenever it changes any of the $r =
|\mathcal{D}| + 1$ tokens the checksum reads, so every added tap makes
every checksum easier to break.

\begin{theorem}[Edit robustness]
\label{thm:robust}
Suppose that an attacker corrupts each token with probability $\rho$,
independently across tokens and of the text, that a corrupted token's
color is uniform on $\mathbb{Z}_q$ and independent of everything else,
that positions stay aligned, and that $\gcd(a_\delta, q) = 1$ for every $\delta \in \mathcal{D}$
(which holds automatically when $q$ is prime). Then, exactly,
$\mathbb{E}[z_{\mathrm{att}}] = (1 - \rho)^{|\mathcal{D}| + 1}\,
\mathbb{E}[z_{\mathrm{clean}}]$, where $z_{\mathrm{att}}$ and
$z_{\mathrm{clean}}$ are the scores of the attacked and the clean text.
\end{theorem}

Random substitution thus multiplies the expected score by the
probability that a whole checksum is untouched (proof in
\Cref{app:proof-robust}). The gcd condition handles a tap that was
edited. Such a tap contributes $a_\delta$ times a uniform residue.
Multiplication by $a_\delta$ permutes the residues exactly when
$a_\delta$ is a unit modulo $q$, and the checksum then holds with
probability exactly $1/q$. Every added tap costs another factor of
$(1-\rho)$, and the ablation of \Cref{sec:experiments-sweeps} shows that
detection under edits falls as taps are added.

\section{Experiments}
\label{sec:experiments}

\subsection{Setup}
\label{sec:experiments-setup}

\textbf{TANGO configuration.} Unless noted otherwise, TANGO uses the
semantic-quantile coloring with $q = 3$ colors, a single tap at lag $\delta=2$
($a_2 = 1$), target residue $b = 0$, logit bias $\beta = 5$, and
enforcement from either side, with the unwhitened coloring on both models.
\Cref{app:protocol} gives the rule that selected these defaults.

\textbf{Models and baselines.} We evaluate LLaDA-8B-Instruct
\citep{nie2025llada} and Dream-v0-Instruct-7B \citep{dream2025}, which is
initialized from Qwen2.5-7B \citep{qwen25} and shares its tokenizer. Prompts
come from the C4 dataset \citep{raffel2020t5}, and TANGO and the baselines generate $128$ tokens at temperature $1.0$, with classifier-free guidance (CFG) at
scale $2.0$ on LLaDA (\Cref{app:protocol}). The baselines are the
red--green list (the Unigram watermark of \citet{zhao2023provable}, with
green fraction $0.25$ and logit bias $5$) and the Gumbel rule keyed on the
preceding token \citep{aaronson2023watermark}, both applied along the
confidence-ordered unmasking. Dream's sampler is incompatible with the
Gumbel rule, so we run the Gumbel rule on LLaDA only. Each comparison in \Cref{tab:headline} uses $n = 200$
generations per method on LLaDA and $n = 100$ per method on Dream. We also
run two published watermarks for masked-diffusion models on the same
prompts, using their public code. They are a context-hashed green list for
diffusion models by \citet{wmdlm2025}, which we call the DLM watermark, and dgMARK
\citep{dgmark2026}.

\textbf{Attacks and metrics.} Edits are word-level
attacks: deleting $30\%$ of the words (del30),
replacing $30\%$ of them with synonyms (syn30), and inserting $20\%$
random words (ins20). The deletion and synonym attacks are adapted from
MarkLLM \citep{pan2024markllm}, and synonyms come from WordNet
\citep{miller1995wordnet}. We report the area under the
detector's ROC curve (AUROC) and the true-positive rate (TPR) at a $1\%$
false-positive rate (FPR), written TPR@1\%FPR, both against unwatermarked
text. Unless stated otherwise, generation quality is measured as perplexity (PPL) under
Qwen2.5-7B-Instruct \citep{qwen25} conditioned on the prompt
(\Cref{app:protocol} gives this protocol and the full attack set). \Cref{app:seeds} reports
the variation over three seeds and four keys.

\subsection{Detection and q{}uality}
\label{sec:experiments-headline}

\Cref{tab:headline} compares TANGO with the red--green list and the
Gumbel rule. TANGO detects $97\%$ (LLaDA) and $100\%$ (Dream) of
unedited watermarked texts, and $68$--$92\%$ and $73$--$98\%$ after
edits, with deletion the hardest attack. Its perplexity is $7.2$ on LLaDA and
$7.8$ on Dream, against $7.6$ and $9.1$ for the red--green list. On
Dream, the red--green list degenerates into repetition on $62\%$ of its
texts, against $31\%$ for TANGO and $20\%$ for unwatermarked text
(\Cref{app:dream}). \Cref{app:samples} shows qualitative samples.

\begin{table}[tbp]
\centering
\caption{TANGO detects nearly all unedited and most edited watermarked
texts at a perplexity close to the red--green list's. TPR@1\%FPR under no
attack (clean) and under the three edits, with the half-width of a $95\%$
bootstrap interval as superscript. TANGO at the default setting, red--green
list with green fraction $0.25$ and bias $5$, Gumbel rule at temperature
$\tau = 1.0$. $n = 200$
per method on LLaDA and $100$ on Dream, with the unwatermarked
continuations of the same prompts as negatives. PPL is Qwen2.5-7B-Instruct
perplexity over non-degenerate continuations (\Cref{app:protocol}).}
\label{tab:headline}
\footnotesize
\setlength{\tabcolsep}{3.5pt}
\begin{tabular}{@{}llrccccr@{}}
\toprule
 & & & \multicolumn{4}{c}{TPR@1\%FPR} & \\
\cmidrule(lr){4-7}
model & method & AUROC & clean & del30 & syn30 & ins20 & PPL \\
\midrule
LLaDA & TANGO    & 0.998 & \ci{0.97}{0.02} & \ci{0.68}{0.07} & \ci{0.92}{0.04} & \ci{0.91}{0.04} & 7.2 \\
LLaDA & Red--green list & 0.993 & \ci{0.98}{0.02} & \ci{0.96}{0.02} & \ci{0.94}{0.03} & \ci{0.98}{0.02} & 7.6 \\
LLaDA & Gumbel   & 0.815 & \ci{0.15}{0.05} & \ci{0.10}{0.04} & \ci{0.07}{0.04} & \ci{0.07}{0.04} & 4.6 \\
LLaDA & unwatermarked & \na & \na & \na & \na & \na & 4.2 \\
\midrule
Dream & TANGO    & 1.000 & \ci{1.00}{0.00} & \ci{0.73}{0.08} & \ci{0.98}{0.03} & \ci{0.97}{0.03} & 7.8 \\
Dream & Red--green list & 1.000 & \ci{1.00}{0.00} & \ci{0.98}{0.03} & \ci{0.98}{0.03} & \ci{0.98}{0.03} & 9.1 \\
Dream & unwatermarked & \na & \na & \na & \na & \na & 3.8 \\
\bottomrule
\end{tabular}
\end{table}

Of the three edits, deletion is the one where the red--green list keeps
a clear lead
($96\%$ against $68\%$ of deleted texts on LLaDA; \Cref{tab:e7} in
\Cref{app:results} sweeps both biases). A deleted word breaks every
TANGO pair that spans it, whereas the red--green list scores each token on
its own. That list's robustness comes from applying the same vocabulary
bias at every position, which is also what exposes its key
(\Cref{sec:experiments-stealth}). Of the published diffusion watermarks,
the DLM watermark lies on the red--green list's trade-off between
detection and perplexity, and dgMARK without its beam lookahead, in its best-detecting setting, detects fewer clean texts than TANGO at a larger perplexity cost (\Cref{tab:external} in \Cref{app:results}).
The Gumbel rule is distortion-free in expectation over its key and costs
little perplexity. It detects poorly, however, with AUROC $0.815$ and
clean TPR@1\%FPR $0.15$, likely because the sampler commits the most confident positions first, where the keyed draw seldom changes the token (\Cref{sec:intro}). Raising the
temperature restores detection at the price of quality. At temperature
$1.5$ it detects $83\%$ of deleted texts, at perplexity $20.6$
(\Cref{tab:e7}).

\subsection{Key recovery and forgery}
\label{sec:experiments-stealth}

\textbf{Key recovery and forgery from token frequencies.} Token frequencies reveal the red--green list's key but not TANGO's. On LLaDA with one-sided enforcement at $\beta = 8$, the total-variation distance between
watermarked and unwatermarked token frequencies is $0.217$ for TANGO. This
is close to the $0.210$ that finite samples alone produce, and below the
$0.340$ of the red--green list (\Cref{tab:stealth} in \Cref{app:results}).
Ranking tokens by how much more often they appear in watermarked than in
unwatermarked text recovers the green list, and the recovery improves as
the attacker collects more texts, from AUC $0.738$ at $25$ texts to
$0.806$ at $100$ (\Cref{fig:steal}, a separate run with TANGO at $\beta =
6$ and the whitened coloring). Against TANGO, recovery stays at chance
at every number of texts ($0.504$ at $200$ texts on LLaDA and $0.501$ at $100$ texts on
Dream). An attacker can then generate text that favors the tokens of the
recovered key. Every such forgery passes the red--green detector, and none
passes TANGO's (\Cref{tab:stealth}).

\textbf{Key recovery and forgery from token pairs.} An attacker who knows
TANGO's design can instead count token \emph{pairs} at the tap lag and
rank them by how much more often they appear in watermarked text. This
attack is much harder, because it must estimate the frequencies of
$O(V^2)$ pairs instead of $O(V)$ tokens, but the attacker may still find
some signal. Scored against the true key, the ranking separates pairs
that satisfy the checksum from the rest with AUC $0.62$ at $200$ texts,
and $0.78$ on the pairs seen at least three times. Turning that signal into a forgery succeeds less often than against the red--green
list. We let an attacker who has the model but not the key collect the over-represented pairs from $200$ texts (against the red--green list, the recovered green list) and bias its own generation toward them. We score forgeries at each detector's $1\%$ false-positive
threshold. At the bias level that lets $51\%$ of forgeries pass the
red--green detector, only $4\%$ pass TANGO's
(\Cref{tab:forge-fluent} in \Cref{app:results}). At this bias the
forgeries against TANGO stay fluent, with GPT-2-large \citep{radford2019gpt2} perplexity $13.5$ against $13.1$
for unwatermarked text. At twice the bias, $55\%$ pass TANGO's detector
and $98\%$ pass the red--green detector. A simpler attacker that strings
over-represented pairs together without a language model shows the same
ordering at every number of texts ($74\%$ against $96\%$ at $200$ texts and $z > 4$,
\Cref{tab:steal-pairs}). Increasing the number of
taps may further weaken both attacks, at the cost of lower detection,
especially under edits (\Cref{sec:experiments-sweeps}).

\subsection{Ablations}
\label{sec:experiments-sweeps}

\Cref{tab:ablations} varies TANGO's two main design choices, the number
of colors $q$ and the tap set $\mathcal{D}$, on LLaDA. The target residue
$b = 0$ is the one that unwatermarked text satisfies least often under the
default key, because the class masses are unequal and the colors of
neighboring tokens are correlated
(\Cref{app:residue-diag}).

\begin{table}[t]
\centering
\caption{Three colors and a single tap detect best under edits.
LLaDA-8B, one-sided enforcement at $\beta = 8$, $n = 24$ per row, defaults
in bold. The $q$ rows use lag $1$ (and, for $q \geq 4$, the key's residue), and the tap rows use $q = 3$, so they are separate runs. PPL is GPT-2-large perplexity, and null FPR is the
fraction of $200$ natural C4 texts with $z > 4$. The residue ablation is
in \Cref{app:residue-diag}.}
\label{tab:ablations}
\footnotesize
\setlength{\tabcolsep}{4pt}
\begin{tabular}{@{}llrccccrr@{}}
\toprule
 & & & \multicolumn{4}{c}{TPR@1\%FPR} & & \\
\cmidrule(lr){4-7}
axis & setting & AUROC & clean & del30 & syn30 & ins20 & PPL & null FPR (\%) \\
\midrule
colors $q$ & $2$ & 0.989 & 0.92 & 0.62 & 0.71 & 0.83 & 20.4 & 0.0 \\
 & $\mathbf{3}$ & \textbf{1.000} & \textbf{1.00} & \textbf{0.92} & \textbf{0.96} & \textbf{1.00} & \textbf{22.2} & \textbf{0.0} \\
 & $4$ & 1.000 & 1.00 & 0.79 & 0.83 & 0.96 & 26.8 & 0.5 \\
 & $8$ & 1.000 & 1.00 & 0.83 & 0.83 & 1.00 & 29.9 & 0.5 \\
 & $16$ & 0.998 & 0.96 & 0.88 & 0.96 & 0.96 & 29.5 & 0.0 \\
\midrule
taps $\mathcal{D}$ & $\boldsymbol{\{2\}}$ & \textbf{1.000} & \textbf{1.00} & \textbf{0.92} & \textbf{0.96} & \textbf{0.96} & \textbf{25.0} & \textbf{0.0} \\
 & $\{1,2\}$ & 0.990 & 0.96 & 0.25 & 0.58 & 0.79 & 19.6 & 0.0 \\
 & $\{1,2,3\}$ & 0.986 & 0.83 & 0.25 & 0.25 & 0.54 & 16.7 & 0.0 \\
 & $\{1,2,3,4\}$ & 0.984 & 0.83 & 0.12 & 0.38 & 0.42 & 18.9 & 0.0 \\
\bottomrule
\end{tabular}
\end{table}

\textbf{Colors and taps.} Detection under edits improves sharply from
$q = 2$ to $q = 3$ (del30 TPR $0.62 \to 0.92$) and does not improve
further, while GPT-2-large perplexity rises from $22.2$ at $q = 3$ to $29.9$ at $q =
8$, so we use $q = 3$. Consistent in direction with \Cref{thm:robust}, deletion TPR drops from $0.92$ with one tap to $0.12$ with four (\Cref{fig:taps} in
\Cref{app:design}), so we use a single tap.

\textbf{Enforced fraction.} The enforced fraction $\kappa$ sets the
margin of \Cref{prop:ez}. With one-sided enforcement at $\beta = 8$,
$\kappa$ is $0.61$ on LLaDA. Either-side enforcement raises it to $0.74$
at the same bias, and to $0.76$ at the default $\beta = 5$
(\Cref{tab:e9-screen} in \Cref{app:design}). Because $\varepsilon =
\kappa\,\bar\varepsilon_{\mathrm{enf}}$, raising $\kappa$ from $0.61$ to
$0.74$ at a fixed margin per enforced pair raises the expected score by a
factor of about $1.2$.

\section{Related work}
\label{sec:related}

\textbf{Biased and distortion-free watermarks.}
Statistical text watermarks differ in whether they change token
frequencies. Biased watermarks do. A keyed green list adds a logit bias to
a vocabulary subset that is either derived from a hash of the context
\citep{kirchenbauer2023watermark} or fixed once \citep{zhao2023provable}.
The resulting frequency shift lets an attacker approximate the green list
and forge or remove the watermark \citep{jovanovic2024stealing,zhang2024mip}.
Reweighting watermarks change the next-token
distribution under a fixed key but preserve it in expectation over keys
\citep{hu2024unbiased,dipmark2024}, and distortion-free watermarks use the
key to seed the randomness of sampling \citep{aaronson2023watermark,
kuditipudi2023robust,christ2023undetectable,dathathri2024synthid}. Like
biased watermarks, TANGO biases the sampler, so detection counts matches
and needs no prompt. Unlike the red--green list, it favors a different class at each
position. Under the assumptions of \Cref{thm:unforgeable}, expected token frequencies are then unchanged, and the signal appears in pair
statistics.

\textbf{Watermarks for diffusion language models.}
The DLM watermark of \citet{wmdlm2025} applies a context-hashed
green-list bias in expectation over still-masked context, and dgMARK
\citep{dgmark2026} steers the unmasking order so that the final text
carries a statistical signature. Other diffusion watermarks add
predicted and bidirectional context to the context hash
\citep{dmark2025}, apply Gumbel-max sampling at every denoising step with
independent \citep{bagchi2025ddlm} or locally correlated
\citep{li2026saccopula} perturbations, or control a global sketch of the
whole sequence \citep{zhao2026sketch}. The closest work, LR-DWM
\citep{raban2026lrdwm}, hashes the token id of the left and of the right
neighbor, whenever each is unmasked, into two green lists and biases the
token toward both. TANGO also keys on an unmasked neighbor, but only
through its color, so the favored set is always one of $q$ balanced
classes, and a substitution that keeps the neighbor's color keeps the
favored class. \Cref{app:related} discusses these works in more detail,
along with semantic watermarks, watermark stealing, and signatures.

\enlargethispage{\baselineskip}
\section{Discussion}
\label{sec:discussion}

Most text watermarks assume left-to-right generation, while
masked-diffusion language models generate tokens in no fixed order. The
red--green list still runs in this setting, but it exposes its key
through token frequencies. TANGO embeds its signal in pairs of tokens
during the denoising loop. Under our idealized assumptions, the expected frequency of every token is then unchanged (\Cref{thm:unforgeable}). In
our experiments, key recovery from token frequencies stays at chance, and
forgeries built from those frequencies fail. Detection needs only the
candidate text and the key.

\textbf{Scope of the guarantees.}
\Cref{thm:unforgeable} covers expected token frequencies under idealized assumptions. An attacker who counts token pairs can forge TANGO,
though less often than the red--green list in every setting we test.
In a word-level test, even an attacker who holds the key and uses it to
choose which words to regenerate removes the watermark about as fast as
random editing does, or slower (\Cref{tab:b3} in \Cref{app:results}).
Rotating the key per deployment window limits how many texts an attacker
can collect under one key, at the cost of testing each candidate text
against every active key. Key rotation is the mitigation we recommend in
practice.

\textbf{Robustness and calibration.}
TANGO trades some robustness to deletion for a frequency-hidden key. The red--green
list survives deletion better at matched perplexity because it biases the
same tokens at every position, which is the property that exposes its key.
\Cref{thm:robust} covers only edits that keep positions aligned.
Paraphrase by back-translation reorders clauses and lowers TANGO's detection more than the red--green list's (\Cref{app:paraphrase}). Dream's natural text
scores above chance, so the detection threshold should be set on
unwatermarked text for each new model (\Cref{sec:experiments-null}).

\textbf{Which watermark to deploy.}
A provider whose detections must hold up against forgery gains the most
from TANGO. Under synonym substitution and word insertion, TANGO detects
at most eight percentage points fewer watermarked texts than the
red--green list on both models (\Cref{tab:headline}). Its key could not be read from token frequencies in our
experiments, and every attack we test forges it less often than the
red--green list. For such a provider, we argue that a frequency-hidden key is worth some loss of robustness to deletion.

\clearpage

\bibliography{iclr2025_conference}
\bibliographystyle{iclr2025_conference}

\clearpage

\appendix
\crefalias{section}{appendix}\crefalias{subsection}{subappendix}\crefalias{subsubsection}{subsubappendix}\etocsettocdepth.toc{subsection}

\etocsettocstyle
  {\begin{center}{\Large\bfseries\color{tocaccent}Appendix Contents}\end{center}\vspace{0.3em}\hrule height 0.8pt\vspace{0.8em}}
  {\vspace{0.8em}\hrule height 0.8pt}
\etocsetstyle{section}
  {}
  {\leftskip 0pt}
  {\smallskip\noindent\textbf{\color{tocaccent}\etocnumber\hspace{0.5em}\etocname}
   \nobreak\hfill\nobreak\etocpage\par}
  {}
\etocsetstyle{subsection}
  {}
  {\leftskip 1.5em}
  {\noindent\etocnumber\hspace{0.5em}\etocname
   \nobreak\hfill\nobreak\etocpage\par}
  {}
\tableofcontents

\newpage
\noindent \Cref{app:proofs} proves the formal results of \Cref{sec:method}
and states every assumption they use. \Cref{app:protocol} gives the
decoding and evaluation settings and the rule that chose TANGO's defaults.
\Cref{sec:experiments-null} checks how well the detection threshold controls false positives on unwatermarked text for both models and explains where
such text departs from the null. \Cref{app:results} gives the full tables
behind \Cref{sec:experiments}, including the Dream results, the published
diffusion watermarks, the attacks that the main text only mentions, and
checks over paraphrase, seeds, and keys. \Cref{app:design} gives the full ablations. \Cref{app:related}
covers further related work, and \Cref{app:samples} shows generated text.

\FloatBarrier
\section{Proofs}
\label{app:proofs}

This appendix proves the four formal statements of \Cref{sec:method} and
eight supporting results. \Cref{tab:results-map} lists every result, what
it says, and the assumptions it uses. \Cref{app:setting} first restates the
notation and the assumptions, so that the appendix can be read without
returning to the main text.

\begin{table}[htbp]
\centering
\caption{The formal results, what each says, and what each assumes. The
assumptions (S), (U), (U$'$), (B), and (I) are stated in \Cref{app:setting}.
Results marked $^\dagger$ also treat a committed token as a draw from its
candidate distribution, the second part of (S).}
\label{tab:results-map}
\footnotesize
\setlength{\tabcolsep}{4pt}
\begin{tabular}{@{}lp{0.46\linewidth}p{0.25\linewidth}l@{}}
\toprule
result & claim in words & assumes & appendix \\
\midrule
\Cref{prop:null} & Unwatermarked text gives a $\mathrm{Binom}(T, 1/q)$ count. & uniform colors, (I) & \ref{app:proof-null} \\
\Cref{prop:keyconc} & A uniformly random balanced coloring gives every class mass close to $1/q$ unless the context is peaked. & hash coloring & \ref{app:proof-null} \\
\Cref{prop:ez} & The expected score is $\sqrt{T}\,\varepsilon / \sigma_0$. & definitions only, fixed $T$ & \ref{app:proof-detect} \\
\Cref{cor:samples} & \raggedright The expected score reaches $z_\star$ once $T \geq T_\star = z_\star^2 \sigma_0^2 / \varepsilon^2$. & $z_\star > 0$, constant $\varepsilon > 0$ & \ref{app:proof-detect} \\
\Cref{lem:factor}$^\dagger$ & \raggedright Unenforced pairs match at the chance rate, so $\varepsilon = \kappa\, \bar\varepsilon_{\mathrm{enf}}$. & \raggedright (S), (U$'$), (B), single tap, $\gcd(a_\delta, q) = 1$ & \ref{app:proof-detect} \\
\Cref{cor:detect} & The false-negative rate decays exponentially in $T$. & (I) & \ref{app:proof-tpr} \\
\Cref{prop:optimal}$^\dagger$ & With one-sided enforcement, the count test is the likelihood-ratio test. & Approximations~1--3 & \ref{app:optimal} \\
\Cref{lem:mixture} & Averaging the $q$ tilts returns the context. & (B) & \ref{app:proof-unforge} \\
\Cref{thm:unforgeable} & For a fixed coloring, expected token frequencies do not change. & (S), (U), (B) & \ref{app:proof-unforge} \\
\Cref{prop:leak} & Without balance, the frequency shift is bounded by the imbalance. & fixed context & \ref{app:proof-unforge} \\
\Cref{prop:greensep} & The red--green list shifts frequencies in the direction of its key. & fixed contexts & \ref{app:proof-unforge} \\
\Cref{thm:robust} & Random substitutions scale the expected score by $(1 - \rho)^{|\mathcal{D}| + 1}$. & \raggedright random substitutions, $\gcd(a_\delta, q) = 1$ & \ref{app:proof-robust} \\
\bottomrule
\end{tabular}
\end{table}

\FloatBarrier
\subsection{Setting, notation, and assumptions}
\label{app:setting}

This subsection restates the checksum, the tilt, and the four assumptions
that the proofs below use. \Cref{tab:notation} collects the symbols.

\begin{table}[htbp]
\centering
\caption{Notation used in the proofs.}
\label{tab:notation}
\footnotesize
\setlength{\tabcolsep}{4pt}
\begin{tabular}{@{}lp{0.72\linewidth}@{}}
\toprule
symbol & meaning \\
\midrule
$\mathcal{V}$, $V$ & vocabulary and its size, $V = |\mathcal{V}|$ \\
$t_1, \dots, t_L$ & tokens of a text of length $L$ \\
$x^{(j)}$ & state after $j$ denoising steps, the partly unmasked sequence \\
$q$, $\mathbb{Z}_q$ & number of colors, and $\{0, \dots, q-1\}$ with arithmetic modulo $q$ \\
$\chi$ & balanced coloring $\mathcal{V} \to \mathbb{Z}_q$ fixed by the key; class $s$ has $n_s \in \{\lfloor V/q \rfloor, \lceil V/q \rceil\}$ tokens \\
$\mathcal{D}$, $\delta$, $a_\delta$ & set of tap lags, one lag, and its coefficient in $\mathbb{Z}_q \setminus \{0\}$ \\
$b$ & target residue in $\mathbb{Z}_q$, fixed for each key \\
$R_i$, $M_i$ & residue and match indicator at scored position $i$ \\
$\mathcal{T}$, $T$ & set of scored positions and its size (the number of scored pairs) \\
$C$, $z$, $z_\star$ & match count, detection score, and detection threshold \\
$\sigma_0^2$ & $\tfrac{1}{q}(1 - \tfrac{1}{q})$, the variance of one match under the null \\
$\beta$, $\gamma$ & logit bias and $\gamma = e^{\beta} - 1$ \\
$p_i$, $m_s(p)$ & context (model distribution) at position $i$, and the class mass, the total probability $p$ assigns to tokens of class $s$ \\
$\mathsf{T}_s p$ & tilt of $p$ toward class $s$ \\
$s_i$ & favored class at position $i$ (for a real key, the class that completes its checksum, \Cref{eq:satisfying}) \\
$\kappa$ & enforced fraction, the expected fraction of scored pairs that are enforced \\
$\mu_i$, $\varepsilon$ & match probability at position $i$ and coupling margin \\
$\bar\varepsilon_{\mathrm{enf}}$ & excess match probability at an enforced pair under balance \\
$\Delta f(v)$ & frequency shift of token $v$ between watermarked and unwatermarked generations \\
$h_j$, $P_j$, $P^0_j$ & history $(x^{(0)}, \dots, x^{(j)})$, and its law with and without the watermark \\
$\bar Z$ & $1 + \gamma/q$, the normalizer of a tilt under balance \\
$\rho$ & substitution rate of the edit attack \\
$A_i$ & positions whose tokens the checksum at position $i$ reads \\
\bottomrule
\end{tabular}
\end{table}

\paragraph{The checksum.}
A secret key fixes a balanced coloring $\chi : \mathcal{V} \to \mathbb{Z}_q$ of the vocabulary into $q$ classes, and we fix a target residue $b \in \mathbb{Z}_q$ for each key. Here $\mathbb{Z}_q = \{0, \dots, q-1\}$ with addition
and multiplication modulo $q$, and \emph{balanced} means that every class
has $\lfloor V/q \rfloor$ or $\lceil V/q \rceil$ tokens. The edit-robustness
result and the tap ablation of \Cref{app:design} use several taps, so the
proofs use a general checksum with a set $\mathcal{D}$ of tap lags and
coefficients $a_\delta \in \mathbb{Z}_q \setminus \{0\}$. The main text's
single tap is the case $\mathcal{D} = \{2\}$ with $a_2 = 1$. In a text $t_1,
\dots, t_L$, position $i$ is scored when all its taps exist, and
$\mathcal{T}$ is the set of scored positions, $T = |\mathcal{T}|$. For a
scored position $i$, the residue and the match indicator are
\[
R_i \;=\; \Big( \chi(t_i) + \sum_{\delta \in \mathcal{D}} a_\delta\,
\chi(t_{i-\delta}) \Big) \bmod q ,
\qquad
M_i \;=\; \mathbf{1}[R_i = b] ,
\]
where $\mathbf{1}[\cdot]$ is the indicator. The detector computes
\[
C \;=\; \sum_{i \in \mathcal{T}} M_i ,
\qquad
z \;=\; \frac{C - T/q}{\sqrt{T}\, \sigma_0} ,
\qquad
\sigma_0^2 \;=\; \tfrac{1}{q}\big(1 - \tfrac{1}{q}\big) ,
\]
as in \Cref{eq:detector-z}, and flags the text when $z > z_\star$. Here
$\sigma_0^2$ is the variance of one match under the null.

\paragraph{The tilt.}
Let $p_i$ be the model's distribution over the vocabulary at position $i$
after temperature scaling. For brevity, in these proofs we call such a
distribution a \emph{context}, since it is what the model predicts from the
tokens around position $i$. Let $m_s(p) = \sum_{v : \chi(v) = s} p(v)$ be
the total probability that a distribution $p$ assigns to tokens of class
$s$. We call $m_s(p)$ the \emph{class mass} of $s$. Adding
$\beta > 0$ to the logits of class $s$ multiplies the probability of each of
its tokens by $e^{\beta}$ and renormalizes. With $\gamma = e^{\beta} - 1$ the
result is the tilt
\[
(\mathsf{T}_s p)(v) \;=\; \frac{p(v)\, e^{\beta \mathbf{1}[\chi(v) = s]}}{1 +
\gamma\, m_s(p)} .
\]
Once the taps of position $i$ are committed, exactly one class completes
its checksum,
\[
s_i \;=\; \Big( b - \sum_{\delta \in \mathcal{D}} a_\delta\,
\chi(t_{i-\delta}) \Big) \bmod q ,
\]
which is \Cref{eq:satisfying} for a general tap set. The tilt favors this
class, so we also call it the \emph{favored class}. All TANGO experiments sample at temperature $\tau = 1$. \Cref{alg:bcc} adds $\beta$ to the logits before
dividing by $\tau$, so for general $\tau$ the temperature-scaled
distribution is tilted by $e^{\beta/\tau}$ instead of $e^{\beta}$.

\paragraph{Enforced pairs.}
This paragraph and \Cref{lem:factor} concern a single tap with
$\gcd(a_\delta, q) = 1$, as \Cref{alg:bcc} requires, so that once one
token of a pair is committed, exactly one class completes the pair.
The detector counts pairs, but the sampler tilts tokens, so we need to say
which pairs the tilt acts on. A position is \emph{committed} when the
sampler unmasks it, and its token is then fixed for the rest of
generation. A token is \emph{enforced} when the sampler draws it from a
tilt. A scored pair is \emph{enforced} when the token of the pair that was
committed second was tilted toward this pair, that is, its favored class was computed from the pair's other token.
With one-sided enforcement a token is tilted only toward the pair it forms
with its tap. Either-side enforcement is defined for a single tap, and
there a token can instead be tilted toward the pair it forms with its right
neighbor at lag $\delta$. Each tilted token completes exactly one pair, so
tilted tokens and enforced scored pairs correspond one to one. The one
exception is a token whose tap lies in the prompt. \Cref{alg:bcc} tilts it,
but the detector does not score its pair. There are at most $\max
\mathcal{D}$ such tokens per text, and we ignore them. Whether a pair is
enforced depends on the unmasking order, so it is random. The
\emph{enforced fraction} is the expected fraction of scored pairs that are
enforced,
\[
\kappa \;=\; \frac{1}{T} \sum_{i \in \mathcal{T}} \Pr(\text{pair $i$ is
enforced}) ,
\]
and \Cref{tab:e9-screen} reports its empirical value.

\paragraph{Assumptions.}
The results use four idealizations of the sampler and the text, and a
local form (U$'$) of the second.
\begin{itemize}
\item \textbf{(S) Sampling.} At each denoising step the sampler draws a
candidate at every masked position, independently given the current
\emph{state} (the partly unmasked sequence) and the favored classes of
that step. An enforced candidate is drawn
from $\mathsf{T}_{s_i} p_i$ and every other candidate from $p_i$. The
sampler then commits candidates by the same \emph{commit rule} as the
unwatermarked sampler. In \Cref{alg:bcc} the commit rule keeps the $K$
most confident candidates. (S) is exact for the candidate draw. For the
commit rule it is an idealization, because the confidences come from the
biased logits and so differ from those of the unwatermarked sampler.

\emph{Committed tokens.} \Cref{lem:factor} and \Cref{app:optimal} also
treat a committed token as a draw from its candidate distribution. This is
exact when the commit rule ignores the candidates, and an idealization
otherwise. \Cref{thm:unforgeable} does not need it, because, averaged over the favored classes, both processes share the candidate distributions and the commit rule.
\item \textbf{(U) Unpredictability.} At each denoising step, given the
history of states $x^{(0)}, \dots, x^{(j-1)}$ and the coloring $\chi$, the
favored classes $s_i$ of the positions enforced in that step are
independent and uniform on $\mathbb{Z}_q$.
\item \textbf{(U$'$) Local unpredictability.} Fix a scored pair $i$, and
let $c$ be the class that completes it once one of its tokens is
committed. A token tilted toward pair $i$ is tilted toward $c$, as in
\Cref{alg:bcc}. Given the history of states, the coloring, and $c$, the
favored class of every token tilted toward a different pair is uniform on
$\mathbb{Z}_q$, and the favored classes of distinct tokens are
independent. In short, (U$'$) is (U) applied to every tilt except the
tilt toward the pair being scored.
\item \textbf{(B) Balance.} Every context spreads its mass evenly over the
classes, $m_s(p_i) = 1/q$ for every $s$ and $i$. \Cref{prop:keyconc} shows
that a uniformly random balanced coloring, such as the hash coloring,
achieves this balance in expectation to within $1/V$. The semantic-quantile
coloring is not uniformly random, so for it (B) is an assumption, and
\Cref{app:residue-diag} measures the class masses under the default key.
\item \textbf{(I) Independence.} The match indicators $M_i$ are mutually
independent across scored positions. (I) is the second hypothesis of
\Cref{prop:null}. Beyond that, it enters only the finite-sample bounds of
\Cref{app:proof-tpr} and Approximation~3 of \Cref{app:optimal}.
\end{itemize}

\paragraph{What (U) idealizes.}
For a real key the favored class $s_i$ is a function of the coloring, the target residue $b$, and the committed taps, so (U) cannot hold exactly. (U) describes
an idealized sampler in which each enforced position draws its favored
class afresh, as if every pair had its own secret residue. Like the real
sampler, it favors a class that changes from position to position. It
also spreads the favored classes evenly over the classes, which the real
sampler does when the colors of the taps are evenly spread. Because the
fresh draw is independent of the coloring, conditioning on the color of
any token leaves it uniform.

For a real key, (U) is therefore an approximation. Its accuracy depends
on how evenly the colors of the taps are spread over the classes and how
unrelated they are to the color of the token drawn at the tilted
position. Natural text meets neither condition exactly
(\Cref{tab:residue-diag}). The red--green list favors the same class at
every position. Its favored class never varies, so no sampler of the kind
(U) describes approximates it.

\paragraph{Why detection uses (U$'$).}
(U) also drops the link between a pair's favored class and the tap of that
pair, and this link is what the detector counts. Under (S), (U), and (B)
the generated text has the unwatermarked law (\Cref{thm:unforgeable}), so
every scored pair matches with probability $1/q$ and the margin
$\varepsilon$ is zero. The reason is that (U) turns the tilt toward a pair
into a tilt toward a uniformly drawn class. With probability $(q-1)/q$
this class is not the one that completes the pair, and the pair then
matches with probability $1/(q + \gamma)$, below chance. Weighted by $(q-1)/q$, this deficit exactly cancels the excess of \Cref{eq:eps-enf}, which occurs with probability $1/q$. (U) and (U$'$)
idealize the same sampler for two different observers. An attacker who
counts single tokens never sees the tap that a token was tilted toward.
Given (B), what matters to this attacker is whether the favored class at a position
is evenly spread and unrelated to the color of the token drawn there,
which is the property that (U) idealizes. The detector reads each token together with its tap, so an analysis of its
score must keep the tilt toward the pair being scored. (U$'$) keeps it
and randomizes only the tilts toward other pairs. We therefore use (U)
for \Cref{thm:unforgeable} and (U$'$) for \Cref{lem:factor}. Neither
describes a real key exactly.

\FloatBarrier
\subsection{The score of unwatermarked text}
\label{app:proof-null}

We call the distribution of the score $z$ on text written without the
watermark the \emph{null}. The threshold $z_\star$ controls the
false-positive rate only if such text satisfies the checksum at rate
$1/q$. \Cref{prop:null} gives conditions under which the count is exactly
binomial, and \Cref{prop:keyconc} shows how close a uniformly random
balanced coloring comes to meeting them.

\noindent\textbf{\Cref{prop:null} (restated).} \emph{Assume that at each
scored position the color $\chi(t_i)$ is uniform on $\mathbb{Z}_q$ and
independent of the colors of its taps, and that the match indicators are
mutually independent across positions. Then $C \sim \mathrm{Binom}(T,
1/q)$ for every residue $b$, and $z$ converges in distribution to
$\mathcal{N}(0,1)$ as $T \to \infty$.}

\begin{proof}[Proof of \Cref{prop:null}]
The proof uses only the two hypotheses and the fact that the scored token's
color enters $R_i$ with coefficient $1$.

\textit{Step 1: each position matches with probability $1/q$.}
Condition on the tap colors of position $i$. They fix the tap contribution
\[
y \;=\; \sum_{\delta \in \mathcal{D}} a_\delta\, \chi(t_{i-\delta}) \bmod q ,
\qquad\text{so that}\qquad
R_i \;=\; \big(\chi(t_i) + y\big) \bmod q .
\]
The map $x \mapsto (x + y) \bmod q$ is a bijection of $\mathbb{Z}_q$ for
every $y$. By hypothesis $\chi(t_i)$ is uniform and independent of the tap
colors, so it stays uniform after conditioning, and a bijection maps a
uniform variable to a uniform variable. Hence $R_i$ is uniform given the
taps, and
\[
\Pr\big( M_i = 1 \mid \text{taps} \big) \;=\; \Pr\big( R_i = b \mid
\text{taps} \big) \;=\; \tfrac{1}{q}
\]
for every residue $b$. The right-hand side does not depend on the taps, so
$\Pr(M_i = 1) = 1/q$.

\textit{Step 2: the count is binomial and the score is asymptotically
normal.} The $M_i$ are independent by hypothesis, so $C \sim
\mathrm{Binom}(T, 1/q)$, with mean $T/q$ and variance $T\sigma_0^2$. The
score $z$ is $C$ standardized by this mean and variance, and the de
Moivre--Laplace theorem gives $z \Rightarrow \mathcal{N}(0, 1)$ as $T \to
\infty$ at fixed $q$.
\end{proof}

The proof does not use the condition $\gcd(a_\delta, q) = 1$. That
condition lets \Cref{alg:bcc} solve \Cref{eq:constraint} for the left token
of a pair, and it handles a corrupted tap in \Cref{thm:robust}.

\paragraph{Where balance enters.}
The hypotheses of \Cref{prop:null} do not mention balance. Balance is what
makes the uniformity hypothesis plausible, because it leaves no color class
larger than another. A token drawn uniformly from the vocabulary lands in
class $s$ with probability $n_s / V$, where $n_s \in \{\lfloor V/q \rfloor,
\lceil V/q \rceil\}$ is the class size, and $|n_s / V - 1/q| < 1/V$. Real
contexts are far from uniform over the vocabulary, so the relevant question
is how much probability a given context $p$ assigns to each class. The next
proposition answers it for a coloring drawn uniformly from the balanced
colorings, which is how the hash coloring of \Cref{sec:method-coloring} is
built.

\begin{proposition}[A random key nearly balances every context]
\label{prop:keyconc}
Let $\chi$ be drawn uniformly from the balanced colorings with class sizes
$n_0, \dots, n_{q-1}$, and fix a distribution $p$ on $\mathcal{V}$. Write
$\|p\|_2^2 = \sum_v p(v)^2$ for the collision probability of $p$, which is
$1/V$ for the uniform distribution and $1$ for a point mass. Then for every
class $s$,
\[
\Big| \mathbb{E}_\chi[m_s(p)] - \tfrac{1}{q} \Big| \;<\; \frac{1}{V},
\qquad
\mathrm{Var}_\chi[m_s(p)] \;=\; \frac{n_s (V - n_s)}{V (V - 1)} \Big(
\|p\|_2^2 - \frac{1}{V} \Big),
\]
and by Chebyshev's inequality, for every $\zeta > 0$,
\[
\Pr\nolimits_\chi\Big( \big|m_s(p) - \mathbb{E}_\chi[m_s(p)]\big| \geq \zeta \Big)
\;\leq\; \frac{\mathrm{Var}_\chi[m_s(p)]}{\zeta^2} .
\]
\end{proposition}

\begin{proof}
Under a uniformly random balanced coloring with fixed class sizes, the set
$\chi^{-1}(s)$ is a uniformly random subset of $\mathcal{V}$ of size $n_s$.
Let $J_v = \mathbf{1}[\chi(v) = s]$, so that $m_s(p) = \sum_v p(v)\, J_v$.

\textit{Step 1: the mean.} Each $J_v$ has mean $n_s/V$, so
$\mathbb{E}_\chi[m_s(p)] = n_s/V$, which is within $1/V$ of $1/q$.

\textit{Step 2: the variance.} Each indicator has $\mathrm{Var}(J_v) =
\tfrac{n_s}{V}(1 - \tfrac{n_s}{V})$, and for $u \neq v$
\[
\mathrm{Cov}(J_u, J_v) \;=\; \frac{n_s (n_s - 1)}{V (V - 1)} -
\frac{n_s^2}{V^2}
\;=\; -\frac{1}{V - 1} \cdot \frac{n_s}{V}\Big(1 - \frac{n_s}{V}\Big) .
\]
The variance of the weighted sum is
\[
\mathrm{Var}_\chi[m_s(p)] \;=\; \sum_v p(v)^2\, \mathrm{Var}(J_v) +
\sum_{u \neq v} p(u)\, p(v)\, \mathrm{Cov}(J_u, J_v) .
\]
Substituting the two moments and using $\sum_{u \neq v} p(u) p(v) = 1 -
\|p\|_2^2$,
\[
\mathrm{Var}_\chi[m_s(p)]
\;=\; \frac{n_s}{V}\Big(1 - \frac{n_s}{V}\Big) \Big[ \|p\|_2^2 - \frac{1 -
\|p\|_2^2}{V - 1} \Big]
\;=\; \frac{n_s (V - n_s)}{V (V - 1)} \Big( \|p\|_2^2 - \frac{1}{V} \Big) ,
\]
where the second equality uses $\|p\|_2^2 - \frac{1 - \|p\|_2^2}{V - 1} =
\frac{V}{V - 1}\big(\|p\|_2^2 - \frac{1}{V}\big)$.

\textit{Step 3: the tail.} Chebyshev's inequality gives the tail bound.
\end{proof}

For $n_s \approx V/q$ the variance is close to $\sigma_0^2\, \|p\|_2^2$. A
random key therefore balances a spread-out context well and a peaked
context poorly. At $q = 3$, a context with collision probability $0.01$ has
class masses with standard deviation about $0.05$ across keys, and a
context that puts almost all its mass on one token has class masses near
$0$ or $1$ for every key.

The proposition does not cover the semantic-quantile coloring. Its classes
are rank buckets of a keyed projection, so a token whose embedding projects
near the median for most directions lands in the middle class for most
keys. For this coloring $\mathbb{E}_\chi[m_s(p)]$ can therefore differ from
$1/q$ by more than $1/V$.

\paragraph{One fixed key.}
\Cref{prop:keyconc} averages over keys, but a deployment uses one key, and
its class masses on natural text can be unequal. Under the default key the
three classes carry $32\%$, $29\%$, and $39\%$ of the tokens of natural text
(\Cref{tab:residue-diag}). This imbalance is why we calibrate the
threshold on unwatermarked text (\Cref{sec:experiments-null}).

\FloatBarrier
\subsection{Expected score and the text length needed for detection}
\label{app:proof-detect}

\Cref{prop:ez} relates the expected score to the coupling margin
$\varepsilon$, which is how far the match probability exceeds the chance
rate $1/q$, averaged over the $T$ scored positions. The proposition needs
no assumption, because it uses only linearity of expectation. Texts can end
early, so $T$ varies across texts. The statements of this subsection then
hold conditionally on $T$.

\noindent\textbf{\Cref{prop:ez} (restated).} \emph{For a fixed number $T$
of scored positions, let $\mu_i = \Pr(M_i = 1)$ under the watermarked
process and $\varepsilon = \frac{1}{T}\sum_{i \in \mathcal{T}}(\mu_i -
1/q)$. Then $\mathbb{E}[z] = \sqrt{T}\,\varepsilon / \sigma_0$.}

\begin{proof}[Proof of \Cref{prop:ez}]
Linearity of expectation gives $\mathbb{E}[C] = \sum_{i} \mu_i = T(1/q +
\varepsilon)$. The score $z$ is an affine function of $C$, so
\[
\mathbb{E}[z] \;=\; \frac{\mathbb{E}[C] - T/q}{\sqrt{T}\, \sigma_0}
\;=\; \frac{T\varepsilon}{\sqrt{T}\, \sigma_0}
\;=\; \sqrt{T}\,\frac{\varepsilon}{\sigma_0} . \qedhere
\]
\end{proof}

Setting the expected score equal to the threshold gives the number of
scored positions that detection needs.

\begin{corollary}[Scored positions to detect]
\label{cor:samples}
If $z_\star > 0$ and the margin $\varepsilon > 0$ does not depend on $T$,
then $\mathbb{E}[z] \geq z_\star$ if and only if
\[
T \;\geq\; T_\star \;=\; \frac{z_\star^2\, \sigma_0^2}{\varepsilon^2} .
\]
\end{corollary}

\begin{proof}
By \Cref{prop:ez}, $\mathbb{E}[z] = \sqrt{T}\,\varepsilon / \sigma_0$
increases in $T$ when $\varepsilon > 0$ and equals $z_\star$ at $T =
T_\star$.
\end{proof}

The number of scored positions that detection needs grows as
$1/\varepsilon^2$, so halving the margin quadruples it.

\paragraph{The margin at an enforced pair.}
The margin depends on how much probability the tilt moves into the favored
class. Consider an enforced pair whose second token has context $p$ and
favored class $s$, and let $m = m_s(p)$. If we treat the committed token as a
draw from $\mathsf{T}_s p$, the pair matches exactly when that token lands
in class $s$, which has probability
\[
m_s(\mathsf{T}_s p) \;=\; \frac{(1 + \gamma)\, m}{1 + \gamma m} .
\]
This exceeds $m$ for every $0 < m < 1$ and equals $m$ at $m \in \{0, 1\}$.
A confident context, whose mass sits almost entirely inside or outside the
favored class, therefore gains almost nothing from the tilt. Under (B) we
have $m = 1/q$, and the excess over chance at an enforced pair is
\begin{equation}
\label{eq:eps-enf}
\bar\varepsilon_{\mathrm{enf}}
\;=\; \frac{(1 + \gamma)/q}{1 + \gamma/q} - \frac{1}{q}
\;=\; \frac{\gamma\,(1 - 1/q)}{q + \gamma} ,
\end{equation}
which increases in $\beta$ and saturates at $1 - 1/q$. Raising $\beta$
therefore increases the margin with diminishing returns.

The main text factors the margin as $\varepsilon = \kappa\,
\bar\varepsilon_{\mathrm{enf}}$. This needs one more fact, that a pair the
sampler did not enforce matches only at the chance rate. The lemma uses
the local assumption (U$'$), because under (U) itself every pair matches
at the chance rate and the margin is zero (\Cref{app:setting}).

\begin{lemma}[Margin factorization]
\label{lem:factor}
Under (S), (U$'$), and (B), for a single tap with $\gcd(a_\delta, q) = 1$,
an enforced pair matches with probability $1/q +
\bar\varepsilon_{\mathrm{enf}}$, with $\bar\varepsilon_{\mathrm{enf}}$ as
in \Cref{eq:eps-enf}, and an unenforced pair matches with probability
$1/q$. Hence $\varepsilon = \kappa\, \bar\varepsilon_{\mathrm{enf}}$.
\end{lemma}

\begin{proof}
Fix a scored pair $i$ and apply (U$'$) to it. Condition on the history of
states and on the coloring, and treat each committed token as a draw from
its candidate distribution (second part of (S)). When the two tokens of
the pair are committed at different steps, the token committed second
decides whether the pair matches, because the other token is already
fixed when it is drawn. Let $c$ be the class that completes the pair given
that other token. The history and the coloring fix $c$.

\textit{Step 1: an enforced pair matches with probability $1/q +
\bar\varepsilon_{\mathrm{enf}}$.} The second token is tilted toward $c$,
the tilt that (U$'$) keeps as in \Cref{alg:bcc}. It lands in $c$ with
probability $m_c(\mathsf{T}_c p)$, which is $1/q +
\bar\varepsilon_{\mathrm{enf}}$ by the computation above with $m = 1/q$
from (B). This holds given every history in which the pair is enforced.

\textit{Step 2: an unenforced pair matches with probability $1/q$.} A token
tilted toward pair $i$ is tilted toward $c$, so the pair is unenforced
exactly when its second token was not tilted toward pair $i$, or when both
tokens were committed in the same step. Three cases cover these outcomes.
\begin{enumerate}
\item[(i)] \emph{The second token was drawn from its context $p$.} It lands in $c$ with
probability $m_c(p) = 1/q$ by (B).
\item[(ii)] \emph{The second token was tilted toward a class $s$ that
completes a different pair.} By (U$'$), $s$ is uniform given the history,
the coloring, and $c$. By (B), the tilt toward $s$ puts mass on $c$ equal
to
\[
m_c(\mathsf{T}_s p) \;=\;
\begin{cases}
\dfrac{1 + \gamma}{q + \gamma} & \text{if } s = c,\\[8pt]
\dfrac{1}{q + \gamma} & \text{if } s \neq c,
\end{cases}
\qquad\text{so}\qquad
\frac{1}{q} \cdot \frac{1 + \gamma}{q + \gamma} + \frac{q - 1}{q} \cdot
\frac{1}{q + \gamma} \;=\; \frac{1}{q} .
\]
\item[(iii)] \emph{Both tokens of the pair were committed in the same
step.} Then neither was tilted toward pair $i$, because a tilt toward a
pair requires its other token to be committed already. Each token was
drawn from its context or from a tilt toward a different pair, whose
favored class is uniform by (U$'$). In both cases its color is uniform, by
(B) and, for a tilt, by \Cref{lem:mixture}. By (S), the two candidates
are independent given the state and the favored classes, and by (U$'$)
the two favored classes are independent. The two colors are therefore
independent and uniform, and
\[
\Pr(R_i = b \mid \text{state})
= \sum_{u \in \mathbb{Z}_q} \Pr\big(\chi(t_{i-\delta}) = u\big)\,
\Pr\big(\chi(t_i) = b - a_\delta u\big)
= \sum_{u \in \mathbb{Z}_q} \frac{1}{q} \cdot \frac{1}{q}
= \frac{1}{q} .
\]
\end{enumerate}

\textit{Step 3: averaging over enforcement.} By Steps 1 and 2,
\[
\mu_i \;=\; \frac{1}{q} + \Pr(\text{pair $i$ is enforced})\,
\bar\varepsilon_{\mathrm{enf}} .
\]
Averaging over $i \in \mathcal{T}$ and using the definition of $\kappa$
gives $\varepsilon = \kappa\, \bar\varepsilon_{\mathrm{enf}}$.
\end{proof}

The factorization explains why enforcing from either side helps. At $\beta =
6$, either-side enforcement raises $\kappa$ from $0.62$ to $0.75$ without changing $\beta$ and therefore without changing $\bar\varepsilon_{\mathrm{enf}}$ (\Cref{tab:e9-screen}).

\FloatBarrier
\subsection{Finite-sample error bounds}
\label{app:proof-tpr}

\Cref{prop:ez} gives the expected score but bounds no error probability.
Under (I) the match indicators are independent variables in $\{0, 1\}$, and
Hoeffding's inequality turns the expected separation into bounds on both
error rates.

\begin{corollary}[Finite-sample error bounds]
\label{cor:detect}
Assume (I). For $T \geq T_\star$ the false-negative rate satisfies
\[
\Pr(z \leq z_\star) \;\leq\;
\exp\!\big({-2}\big(\sqrt{T}\,\varepsilon - z_\star \sigma_0\big)^2\big).
\]
Under the null, where every match probability is $1/q$, the false-positive
rate satisfies, for every $T$,
\[
\Pr(z > z_\star) \;\leq\; \exp\!\big({-2} z_\star^2 \sigma_0^2\big) .
\]
\end{corollary}

\begin{proof}
Under (I), Hoeffding's inequality \citep{hoeffding1963} gives, for every
$\zeta \geq 0$,
\[
\Pr\big(C - \mathbb{E}[C] \leq -\zeta\big) \;\leq\; e^{-2\zeta^2/T},
\qquad
\Pr\big(C - \mathbb{E}[C] \geq \zeta\big) \;\leq\; e^{-2\zeta^2/T} .
\]

\textit{False negatives.} Since $\mathbb{E}[C] = T/q + T\varepsilon$,
\[
\{z \leq z_\star\} \;=\; \big\{C \leq T/q + z_\star \sigma_0 \sqrt{T}\big\}
\;=\; \big\{C - \mathbb{E}[C] \leq -\zeta\big\},
\qquad
\zeta \;=\; T\varepsilon - z_\star \sigma_0 \sqrt{T} .
\]
For $T \geq T_\star$ we have $\zeta \geq 0$, and $\zeta^2 / T =
(\sqrt{T}\varepsilon - z_\star \sigma_0)^2$. The first inequality gives the
bound.

\textit{False positives.} Under the null $\mathbb{E}[C] = T/q$, so
\[
\{z > z_\star\} \;=\; \big\{C - \mathbb{E}[C] > z_\star \sigma_0
\sqrt{T}\big\} .
\]
The second inequality with $\zeta = z_\star \sigma_0 \sqrt{T}$ gives
$\exp(-2 z_\star^2 \sigma_0^2)$.
\end{proof}

For $T \geq T_\star$, the false-negative bound decays exponentially in $T$.
The false-positive bound depends only on $z_\star$ and $q$. At $q = 3$ and
$z_\star = 4$ it is $\exp(-64/9) = 8.2 \times 10^{-4}$. Hoeffding's
inequality ignores the variance of the matches, so when the binomial null
of \Cref{prop:null} holds, the exact binomial tail is smaller.

Rounding in the class sizes moves each null match probability by at most
$1/V$ (\Cref{app:proof-null}), so it moves the null mean of $C$ by at most
$T/V$. Applying Hoeffding's inequality around the shifted mean changes the
false-positive bound by less than $10^{-6}$ at $V \approx 1.26 \times
10^{5}$ and $T \leq 128$.

\FloatBarrier
\subsection{The count test is the likelihood-ratio test}
\label{app:optimal}

The main text motivates the count $C$ as the natural statistic for a
checksum that holds at rate $1/q$ by chance. This subsection shows that,
under three approximations and with one-sided enforcement, thresholding
$C$ is the likelihood-ratio test. By the Neyman--Pearson lemma, that test
detects the largest fraction of watermarked texts among all detectors with
the same false-positive rate.

Given a text $t_1, \dots, t_L$, the optimal detector computes the
likelihood ratio
\begin{equation}
\label{eq:lr}
\Lambda \;=\; \frac{\Pr(\text{text} \mid \text{watermarked model})}
{\Pr(\text{text} \mid \text{no watermark})}
\end{equation}
and flags the text when $\Lambda$ exceeds a threshold. Computing
$\Lambda$ is infeasible. The numerator sums over every prompt and every
unmasking order, both terms require running the model, and for human
authors the denominator is unknown. Three approximations reduce
\Cref{eq:lr} to TANGO's detector. We index each scored pair by its scored
position $i$, so that $M_i$ is its match indicator. For an enforced pair
$i$ (\Cref{app:setting}) we write $m_i = m_{s_i}(p)$ for the probability
that the unwatermarked context $p$ of the pair's tilted token assigns to
the favored class.

\paragraph{Approximation 1: fix the prompt and the unmasking order.}
Assume that the watermarked text and its unwatermarked counterpart come
from the same model, prompt, and unmasking order. The two hypotheses then
differ only at the tilted token of each enforced pair. Consider one
enforced pair $i$ and its tilted token. Token probabilities are a softmax
of logits $\ell_v$, $p(v) = e^{\ell_v}/\Omega$ with $\Omega = \sum_w
e^{\ell_w}$. The watermark adds $\beta$ to the logits of the favored
class. We first compute how this changes the normalizer, because the
likelihood ratio of the token is its boost divided by that change. The new
normalizer $\Omega'$ satisfies
\[
\frac{\Omega'}{\Omega}
\;=\; e^{\beta} \!\!\sum_{w :\, \chi(w) = s_i}\!\! p(w)
\;+\; \sum_{w :\, \chi(w) \neq s_i} p(w)
\;=\; 1 + \gamma\, m_i .
\]
Given the contexts that this approximation fixes, both hypotheses assign
the text a product of per-position token probabilities, and
Approximation~3 reuses this factorization. The log-ratio $\log \Lambda$ is therefore a
sum of per-token log-ratios. A token
that was not tilted contributes $0$, because both hypotheses use the same
distribution there. The tilted token of an enforced pair lands in the
favored class exactly when the pair's checksum holds. Its probability is
multiplied by $e^{\beta}$ in that case and in both cases divided by $1 +
\gamma m_i$, so its log-ratio is
\[
\log \frac{\Pr_{\mathrm{wm}}(t)}{\Pr_{0}(t)} \;=\;
\begin{cases}
\beta - \log(1 + \gamma m_i) & \text{if } M_i = 1,\\
-\log(1 + \gamma m_i) & \text{if } M_i = 0.
\end{cases}
\]
Summing over enforced pairs,
\begin{equation}
\label{eq:lr-boosted}
\log \Lambda \;=\; \sum_{i \text{ enforced}} \Big[ \beta M_i \;-\; \log\big(1 +
\gamma\, m_i\big) \Big].
\end{equation}

\paragraph{Approximation 2: balance.}
The indicators $M_i$ in \Cref{eq:lr-boosted} depend only on the text and
the key, but $m_i$ requires running the model. Under (B), $m_i = 1/q$ at
every position, and every penalty becomes the constant $\log(1 +
\gamma/q)$.

\paragraph{Approximation 3: enforcement at random.}
The final text does not record the order in which positions were unmasked,
so it does not reveal which pairs were enforced. Assume one-sided
enforcement, so that the tilted token of pair $i$ is $t_i$. Assume further
that each scored pair was enforced with probability $\kappa$,
independently across pairs. This independence is itself an approximation,
since pairs $i$ and $i + \delta$ share the token $t_i$. Under the null the
factorization of Approximation~1 makes the colors of different
positions independent and, by (B), uniform, which implies (I). Under the
watermark, the probability of $t_i$ is a mixture of its tilted and its
unwatermarked probability,
\[
\Pr\nolimits_{\mathrm{wm}}(t_i = v) \;=\; \kappa\, (\mathsf{T}_{s_i} p)(v) + (1 -
\kappa)\, p(v) .
\]
Under (B) the normalizer of every tilt is $1 + \gamma/q$. Dividing by
$p(v)$ therefore cancels the token's identity and leaves only its match
indicator, which gives the per-pair likelihood ratio
\begin{equation}
\label{eq:lr-position}
\lambda(M_i) \;=\; \frac{\Pr_{\mathrm{wm}}(t_i = v)}{p(v)} \;=\; \kappa\,
\frac{e^{\beta M_i}}{1 + \gamma/q} \;+\; (1 - \kappa).
\end{equation}
Since $e^{\beta} = 1 + \gamma > 1 + \gamma/q$, \Cref{eq:lr-position}
gives $\lambda(1) > 1$ for a match and $\lambda(0) < 1$ for a miss. Both
likelihoods factor over positions, so $\Lambda$ is the product of the
per-pair ratios, and over the $T$ scored pairs
\begin{equation}
\label{eq:lr-affine}
\log \Lambda \;=\; \underbrace{\log \frac{\lambda(1)}{\lambda(0)}}_{>\,0}
\cdot\, C \;+\; T \log \lambda(0).
\end{equation}

\begin{proposition}[The count test is the likelihood-ratio test]
\label{prop:optimal}
Under Approximations~1--3, which include one-sided enforcement, and among
texts with the same number $T$ of scored positions, $\log \Lambda$ is an
increasing affine function of the match count $C$. Thresholding $\Lambda$
is therefore the same test as thresholding $C$, or equivalently $z$. The
threshold on $C$ that gives a chosen false-positive rate depends only on
the null, so the same count test is the likelihood-ratio test for every
$\beta > 0$ and every $\kappa \in (0, 1]$.
\end{proposition}

\begin{proof}
For fixed $T$, \Cref{eq:lr-affine} is affine in $C$ with slope
$\log(\lambda(1)/\lambda(0)) > 0$. Hence $\Lambda > \lambda_\star$ holds if
and only if $C > C_\star$, for a threshold $C_\star$ determined by
$\lambda_\star$, and the two tests flag the same texts. Under the null of
Approximations~2 and~3, $C \sim \mathrm{Binom}(T, 1/q)$, which does not
involve $\beta$ or $\kappa$. The threshold $C_\star$ that gives a chosen
false-positive rate is therefore the same for every $\beta$ and $\kappa$.
\end{proof}

Because $C$ is discrete, the Neyman--Pearson optimality holds exactly at
the false-positive rates that a threshold on $C$ attains. Other rates need
a randomized threshold.

The proposition does not extend to either-side enforcement. There a token
can be tilted toward either of its two pairs, so its likelihood ratio
depends on both match indicators $M_i$ and $M_{i+\delta}$, and $\log
\Lambda$ is no longer a function of $C$ alone. We use the count test for
both modes.

Approximation~1 fits the question a provider asks about its own output,
namely whether its model applied the bias while generating this text.
Approximations~2 and~3 are (B) and (I), plus one-sided enforcement and the
assumption that every pair is enforced with the same probability
$\kappa$. Under one-sided enforcement the measured $\kappa$ is $0.61$ at
$\beta = 8$ and $0.62$ at $\beta = 6$ (\Cref{tab:e9-screen}). The calibration of \Cref{prop:null} needs neither Approximation~1 nor the enforcement assumptions of Approximation~3, so the calibration does not depend on how the model samples. In practice the threshold is still set per model on unwatermarked text (\Cref{sec:experiments-null}).

\FloatBarrier
\subsection{Token freq{}uencies and the key}
\label{app:proof-unforge}

An attacker who compares token frequencies in watermarked and
unwatermarked text learns the key if the watermark shifts frequencies
differently for different classes. This subsection proves that balanced
tilts cancel in expectation (\Cref{lem:mixture,thm:unforgeable}), bounds
the shift when balance fails (\Cref{prop:leak}), and shows that the
red--green list has no such cancellation (\Cref{prop:greensep}).

The first two results rest on one identity. If a context spreads its
mass evenly over the $q$ classes, then tilting toward each class in turn
and averaging returns the context.

\begin{lemma}[Mixture identity]
\label{lem:mixture}
If $m_s(p) = 1/q$ for every $s$, then $\tfrac{1}{q} \sum_{s=0}^{q-1}
\mathsf{T}_s p \;=\; p$ for every $\beta$.
\end{lemma}

\begin{proof}
Under the hypothesis every normalizer equals $1 + \gamma/q = (q +
\gamma)/q$. Fix a token $v$. Exactly one of the $q$ tilts multiplies $p(v)$
by $e^{\beta}$, and the other $q - 1$ leave it unchanged, so
\[
\frac{1}{q} \sum_{s} (\mathsf{T}_s p)(v)
\;=\; \frac{p(v)}{q} \cdot \frac{q}{q + \gamma}\big( e^{\beta} + q - 1
\big) \;=\; p(v),
\]
since $e^{\beta} + q - 1 = q + \gamma$.
\end{proof}

\noindent\textbf{\Cref{thm:unforgeable} (restated).} \emph{Fix a coloring
$\chi$. Under (S), (U), and (B), the generated text has the same distribution
with and without the watermark. Hence $\mathbb{E}[\Delta f(v)] = 0$ for
every $v \in \mathcal{V}$, and any attacker who ranks each token $v$ by
$g(\mathbb{E}[\Delta f(v)])$, for a fixed function $g$, separates the
color classes no better than chance (AUC $\tfrac{1}{2}$).}

Here the expectation is over the sampler and the favored classes of (U),
with the coloring fixed. The proof must compare two whole generation
processes, not one position. The model conditions on every committed token,
so a watermark that changed which tokens appear early would also change the
contexts of later positions. We therefore show that one denoising step has
the same transition law with and without the watermark, given the whole
history, and conclude by induction.

\begin{proof}[Proof of \Cref{thm:unforgeable}]
Fix the coloring $\chi$ and the prompt. Write $h_j = (x^{(0)}, \dots,
x^{(j)})$ for the history of states after $j$ steps, and $P_j$ and $P^0_j$
for the laws of $h_j$ with and without the watermark.

\textit{Step 1: one step has the same transition law.} Fix a history
$h_{j-1}$ with current state $x = x^{(j-1)}$. By (S), the candidates at the
masked positions are drawn independently given $x$ and the favored
classes of the step, from
$\mathsf{T}_{s_i} p_i$ at enforced positions and from $p_i$ elsewhere,
where the contexts $p_i$ depend only on $x$. By (U), the favored classes of
the enforced positions are independent and uniform given $h_{j-1}$ and
$\chi$. Averaging over them, each enforced candidate has distribution
\[
\frac{1}{q} \sum_{s \in \mathbb{Z}_q} \mathsf{T}_s p_i \;=\; p_i ,
\]
where the equality is \Cref{lem:mixture} and is the only place (B) enters.
Because the favored classes are independent across positions, averaging
over them keeps the candidates independent, so given $h_{j-1}$ they have
the joint distribution $\prod_i p_i$, which is the unwatermarked one. By (S) the commit rule is
also the same, so $x^{(j)}$ has the same law given $h_{j-1}$ in both
processes.

\textit{Step 2: induction over steps.} Both processes start from the fully
masked state, so $P_0 = P^0_0$. If $P_{j-1} = P^0_{j-1}$, then Step 1 shows
that the two processes extend each history $h_{j-1}$ by the same law, so
$P_j = P^0_j$. By induction the laws of the full history agree, and in
particular the final state, which is the generated text, has the same
distribution in both processes. This holds for every prompt.

\textit{Step 3: frequencies and ranking.} Every token therefore has the
same expected frequency in both processes, and $\mathbb{E}[\Delta f(v)] =
0$ for every $v$. The score $g(\mathbb{E}[\Delta f(v)]) = g(0)$ therefore
takes the same value on every token. A constant score ties
every pair of tokens from different classes, and with ties counted as one
half, the AUC is $\tfrac{1}{2}$.
\end{proof}

The theorem is a statement about the idealized sampler of (U), in which the
favored class is a fresh draw at every position (\Cref{app:setting}). It
does not say that one key produces unwatermarked text, which would
contradict detection. What it isolates is that each tilt, averaged over a
favored class that is uniform and unrelated to the token's color, leaves
the distribution unchanged. For a real key that satisfies (B), frequency shifts can therefore arise only where the favored class is related to the color of the token
drawn, which happens when the taps' colors are unevenly spread over the
classes or correlated with the color at the tilted position.

The contrast with the red--green list does not come from this theorem.
That list favors the same class at every position, so the favored class
never varies and no sampler of the kind (U) describes approximates it.
\Cref{prop:greensep} shows that its frequency shifts then separate the
green list from the rest.

The theorem also concerns expected frequencies only. An attacker sees a
finite number of texts, whose empirical frequency shifts are noisy and
need not vanish exactly. The theorem does not cover this attacker, and
\Cref{tab:stealth} measures what it recovers.

\paragraph{Frequency shifts when balance fails.}
\Cref{thm:unforgeable} rests on (B). When a context puts more than $1/q$ of
its mass on some class, the tilts no longer cancel exactly and token
frequencies shift. The next proposition gives the shift in closed form and
bounds it by the imbalance.

\begin{proposition}[Imbalance bounds the frequency shift]
\label{prop:leak}
Fix a context $p$ with class masses $m_s = m_s(p)$ and imbalances $\eta_s =
m_s - 1/q$, and let $\phi(x) = 1/(1 + \gamma x)$. For every token $v$ with
$\chi(v) = c$,
\[
\frac{1}{q} \sum_s (\mathsf{T}_s p)(v) - p(v)
\;=\; p(v)\, \frac{\gamma}{q} \Big[ \phi(m_c) - \sum_s m_s \phi(m_s) \Big] ,
\]
and its magnitude is at most $p(v)\, \tfrac{2\gamma^2}{q} \max_s |\eta_s|$.
To first order in the imbalance, the shift is $-p(v)\, \gamma^2 \eta_c /
(q \bar Z^2)$ with $\bar Z = 1 + \gamma/q$, the normalizer under balance.
\end{proposition}

\begin{proof}
\textit{Step 1: the identity.} The tilt toward $v$'s own class multiplies
$p(v)$ by $(1 + \gamma)\phi(m_c)$, and the tilt toward any other class $s$
multiplies it by $\phi(m_s)$, so
\[
\frac{1}{q} \sum_s (\mathsf{T}_s p)(v)
\;=\; \frac{p(v)}{q} \Big[ (1 + \gamma) \phi(m_c) + \sum_{s \neq c} \phi(m_s)
\Big]
\;=\; \frac{p(v)}{q} \Big[ \gamma \phi(m_c) + \sum_{s} \phi(m_s) \Big] .
\]
From $\phi(x)(1 + \gamma x) = 1$ and $\sum_s m_s = 1$,
\[
\phi(m_s) \;=\; 1 - \gamma\, m_s \phi(m_s)
\quad\Longrightarrow\quad
\sum_s \phi(m_s) \;=\; q - \gamma \sum_s m_s \phi(m_s) .
\]
Substituting into the bracket and subtracting $p(v)$ proves the identity.

\textit{Step 2: the bound.} We compare $\phi$ at $m_c$ with $\phi$ at each
$m_s$ through the mean value theorem. Since $\sum_s m_s = 1$,
\[
\phi(m_c) - \sum_s m_s \phi(m_s) \;=\; \sum_s m_s \big( \phi(m_c) -
\phi(m_s) \big) .
\]
On $[0, 1]$ we have $|\phi'(x)| = \gamma/(1 + \gamma x)^2 \leq \gamma$, so
\[
|\phi(m_c) - \phi(m_s)| \;\leq\; \gamma\, |\eta_c - \eta_s| \;\leq\;
2\gamma \max_s |\eta_s| .
\]
Averaging over $s$ with weights $m_s$ and multiplying by $p(v)\gamma/q$
gives the bound.

\textit{Step 3: the first-order term.} Expanding around balance,
\[
\phi(m_s) \;=\; \phi(1/q) + \phi'(1/q)\, \eta_s + O(\eta^2),
\qquad
\phi'(1/q) \;=\; -\gamma/\bar Z^2 .
\]
Since $\sum_s \eta_s = 0$, we get $\sum_s m_s \phi(m_s) = \phi(1/q) +
O(\eta^2)$, so the bracket equals
\[
\phi'(1/q)\, \eta_c + O(\eta^2) \;=\; -\gamma\, \eta_c / \bar Z^2 +
O(\eta^2) .
\]
Multiplying by $p(v)\gamma/q$ gives the stated term.
\end{proof}

The shift is first order in the imbalance, and for small $\gamma$ it is
second order in $\gamma$. To first order, it depends only on the imbalance
of $v$'s own class. A token in an over-represented class ($\eta_c > 0$)
becomes less frequent under the watermark, and a token in an under-represented class more frequent. At the default $\beta = 5$,
where $\gamma \approx 147$, the two coefficients are
\[
\frac{2\gamma^2}{q} \;\approx\; 1.4 \times 10^{4}
\quad\text{(uniform bound)},
\qquad
\frac{\gamma^2}{q \bar Z^2} \;\approx\; 2.9
\quad\text{(first order)} .
\]
The uniform bound exceeds $p(v)$ at this bias, so the first-order
coefficient is the informative one. It says that a class imbalance of $0.06$
moves a token's expected frequency in one context by about $17\%$ of
$p(v)$, in a direction set by the key. The evidence that this shift does
not reveal the key in practice comes from the key-recovery experiment of
\Cref{tab:stealth}.

\paragraph{The red--green list.}
The red--green list tilts every position toward the same token set $G
\subset \mathcal{V}$, so there is no average over favored classes and
nothing cancels. To isolate this effect, we hold the contexts fixed and
compare the tilted and untilted distributions in each one. Like the tilt toward a
class in \Cref{eq:tilt}, the tilt toward $G$ multiplies the probability of
each token in $G$ by $e^{\beta}$ and renormalizes. Let ${m_G(p) =
\sum_{v \in G} p(v)}$ be the total probability that $p$ assigns to tokens in
$G$. The tilt is
\[
(\mathsf{T}_G p)(v) \;=\; \frac{p(v)\, e^{\beta \mathbf{1}[v \in G]}}{1 +
\gamma\, m_G(p)} .
\]
For contexts $p_1, \dots, p_H$, the average shift is
\[
\Delta_G f(v) \;=\; \frac{1}{H} \sum_{h=1}^{H} \Big[ (\mathsf{T}_G
p_h)(v) - p_h(v) \Big] .
\]

\begin{proposition}[The red--green list shifts frequencies toward its key]
\label{prop:greensep}
If $0 < m_G(p_h) < 1$ for every context, and every token has $p_h(v) > 0$
in at least one context, then
$\Delta_G f(v) > 0$ for every $v \in G$ and $\Delta_G f(v) < 0$ for every $v
\notin G$. Ranking tokens by $\Delta_G f$ therefore separates $G$ from its
complement with AUC $1$.
\end{proposition}

\begin{proof}
In context $p_h$ the normalizer is $Z_h = 1 + \gamma\, m_G(p_h)$, which
lies strictly between $1$ and $1 + \gamma = e^{\beta}$ because $0 <
m_G(p_h) < 1$. The shift in context $h$ is
\[
(\mathsf{T}_G p_h)(v) - p_h(v) \;=\;
\begin{cases}
p_h(v)\, \big( e^{\beta}/Z_h - 1 \big) & \text{if } v \in G, \text{ with }
e^{\beta}/Z_h > 1,\\[2pt]
p_h(v)\, \big( 1/Z_h - 1 \big) & \text{if } v \notin G, \text{ with }
1/Z_h < 1.
\end{cases}
\]
Each shift is therefore nonnegative for $v \in G$ and nonpositive for $v
\notin G$, and it is nonzero in every context with $p_h(v) > 0$. Averaging
over contexts, $\Delta_G f(v) > 0$ on $G$ and $\Delta_G f(v) < 0$ off $G$.
Every token of $G$ then ranks above every token outside $G$, so every pair
from different sets is ordered correctly and the AUC is $1$.
\end{proof}

Holding the contexts fixed isolates the direct effect of the tilt. In real
generation the tilt also changes later contexts, which this statement does
not model. \Cref{fig:steal} shows that the effect is visible in practice.

\FloatBarrier
\subsection{Edit robustness}
\label{app:proof-robust}

\Cref{thm:robust} models an attacker who substitutes tokens at random. It
uses only the following attack model and none of (S), (U), (U$'$), (B), or (I).
\begin{itemize}
\item The attacker draws corruption indicators $X_j \sim
\mathrm{Bern}(\rho)$ independently across positions and independently of
the text.
\item A corrupted token's color is replaced by a variable $\xi_j$ that is
uniform on $\mathbb{Z}_q$ and independent of everything else.
\item The attacker neither inserts nor deletes tokens, so every position
keeps its index.
\item The residue $b$ is fixed, and $\gcd(a_\delta, q) = 1$ for every
$\delta \in \mathcal{D}$.
\end{itemize}

The proof counts which checksums an edit touches. A checksum reads
$|\mathcal{D}| + 1$ tokens. If none of them is corrupted, the checksum
behaves as on clean text. If any is corrupted, it matches at exactly the
chance rate. The excess over chance therefore survives with the probability
that all $|\mathcal{D}| + 1$ tokens are untouched.

\noindent\textbf{\Cref{thm:robust} (restated).} \emph{Under the attack
model above, $\mathbb{E}[z_{\mathrm{att}}] = (1 - \rho)^{|\mathcal{D}| +
1}\, \mathbb{E}[z_{\mathrm{clean}}]$, where $z_{\mathrm{att}}$ and
$z_{\mathrm{clean}}$ are the scores of the attacked and the clean text.}

\begin{proof}[Proof of \Cref{thm:robust}]
Fix a scored position $i$ and the set of positions its checksum reads,
\[
A_i \;=\; \{i\} \cup \{i - \delta : \delta \in \mathcal{D}\} ,
\]
which has $|\mathcal{D}| + 1$ elements. Condition on the corruption pattern
$(X_j)_{j \in A_i}$, and write $\chi'$ for the colors after the attack. We show
that
\[
\Pr\big(M_i = 1 \mid (X_j)_{j \in A_i}\big) \;=\;
\begin{cases}
\mu_i & \text{if no position in $A_i$ is corrupted},\\
1/q & \text{otherwise}.
\end{cases}
\]

\textit{Case 1: no position in $A_i$ is corrupted.} This has probability
$(1-\rho)^{|\mathcal{D}|+1}$, and the residue is computed on clean tokens.
The corruption pattern is independent of the text, so the conditional match
probability is the clean $\mu_i$.

\textit{Case 2: at least one position in $A_i$ is corrupted.} The attacked
residue is
\[
R_i \;=\; \Big( \chi'(t_i) + \sum_{\delta \in \mathcal{D}} a_\delta\,
\chi'(t_{i-\delta}) \Big) \bmod q .
\]
Pick one corrupted position $j \in A_i$. Its term in $R_i$ is $\xi_j$ if $j
= i$, and $a_\delta \xi_j$ if $j = i - \delta$. In the second case
$a_\delta \xi_j$ is uniform, because multiplication by $a_\delta$ permutes
$\mathbb{Z}_q$ when $\gcd(a_\delta, q) = 1$. Write $R_i = (w_j + y_j) \bmod
q$, where $w_j$ is this term and $y_j$ is the sum of the remaining terms,
which involves only the text and the replacement colors of the other
positions. The term $w_j$ is uniform and independent of $y_j$. The sum of a
uniform variable and an independent variable on $\mathbb{Z}_q$ is uniform,
so $R_i$ is uniform and $\Pr(R_i = b) = 1/q$.

\textit{Averaging.} Averaging over the corruption pattern and summing over
positions,
\[
\mathbb{E}[C_{\mathrm{att}}] \;=\; \sum_{i \in \mathcal{T}}
\Big[ (1-\rho)^{|\mathcal{D}|+1} \mu_i +
\big(1 - (1-\rho)^{|\mathcal{D}|+1}\big)\tfrac{1}{q} \Big]
\;=\; \tfrac{T}{q} + (1-\rho)^{|\mathcal{D}|+1}\, T \varepsilon .
\]
Substitutions do not change $T$, so the same $\sqrt{T}\,\sigma_0$
normalizes both scores. Applying \Cref{prop:ez} to the clean and attacked
texts gives $\mathbb{E}[z_{\mathrm{att}}] = (1-\rho)^{|\mathcal{D}|+1}\,
\mathbb{E}[z_{\mathrm{clean}}]$.
\end{proof}

The condition $\gcd(a_\delta, q) = 1$ is used only for a corrupted tap. If
$g = \gcd(a_\delta, q) > 1$, then $a_\delta \xi_j$ takes only the $q/g$
multiples of $g$ in $\mathbb{Z}_q$, and a checksum whose only corrupted
position is that tap need not match at rate $1/q$. When $q$ is prime, every
nonzero coefficient satisfies the condition.

\paragraph{Insertions and deletions.}
These edits shift positions, so an edit at position $j$ changes which
tokens form the pairs that span $j$. The theorem does not model this. For
these edits it predicts the direction of the effect (\Cref{fig:taps}) but
not its size.

\paragraph{Synonyms.}
The theorem gives a corrupted token a uniformly random color. The
semantic-quantile coloring is designed so that a synonym keeps its color
more often than chance. To the extent that it does, the theorem
underestimates how much of the score survives synonym substitution.

\FloatBarrier
\section{Experimental details}
\label{app:protocol}

This appendix gives the settings shared by all experiments, the two
perplexity protocols, the attacks, and the rule that selected TANGO's
defaults. Settings that differ for a single experiment are stated in the
caption of its table.

\paragraph{Decoding.}
All methods generate $128$ tokens in $128$ denoising steps at temperature
$1.0$ with no top-$p$ truncation. On LLaDA every method uses classifier-free guidance (CFG) at scale $2.0$, that is, logits $\ell_{\mathrm{u}} + 2(\ell_{\mathrm{c}} - \ell_{\mathrm{u}})$ from the unconditional and conditional logits, which is $w = 1$ in the notation of \citet{nie2025llada}. The Gumbel rule keys a position on its preceding token when that token is already unmasked and otherwise samples the position without the key, and its detector scores every token. Dream's sampler has no CFG, so Dream runs use
a plain forward pass. Except for the published diffusion watermarks of \Cref{app:results-external}, which are paired with their own controls, these settings are the same for every method on a model, so perplexities are comparable within a model. Coloring and scoring
use each model's full output vocabulary, the rows of its output layer ($V = 126{,}464$ for LLaDA and
$152{,}064$ for Dream).

\paragraph{Perplexity.}
Unless stated otherwise, the main text reports perplexity under Qwen2.5-7B-Instruct in bfloat16,
conditioned on the prompt. We pool the
continuations into one perplexity, the exponential of their total negative
log-likelihood divided by their total number of tokens. Appendix columns
marked PPL GPT-2 report the GPT-2-large \citep{radford2019gpt2}
perplexity of the continuation alone, averaged over texts, and columns
marked PPL Qwen report the pooled Qwen perplexity.

\paragraph{Degenerate continuations.}
Perplexity rewards repetition, because each repeat of a phrase is
near-certain given the earlier copies. We call a continuation
\emph{degenerate} when more than $20\%$ of its whitespace 4-grams repeat an
earlier 4-gram of the same text. In practice this means a phrase or a
number repeated up to the length limit.
\Cref{tab:headline,tab:e7,tab:e4-dream} score perplexity over the
non-degenerate continuations only, and every other appendix table scores
all continuations. On LLaDA the choice matters little. At the settings of
\Cref{tab:headline,tab:e7}, at most $6\%$ of texts are degenerate. The
two perplexities differ by at most $0.5$ at the settings of
\Cref{tab:headline} and by at most $1.3$ across \Cref{tab:e7}. The
ablation settings of \Cref{app:design} reach $12.5\%$. On Dream the
choice matters more. There,
$20\%$ of unwatermarked texts, $31\%$ of TANGO texts, and $62\%$ of
red--green texts are degenerate. Over all continuations, the perplexities
behind \Cref{tab:headline} are $4.1$ (unwatermarked), $7.0$ (TANGO), and
$7.4$ (red--green list) on LLaDA, and $3.5$, $6.2$, and $4.9$ on Dream.
Repeated phrases have low perplexity, so scoring all continuations favors
the method that degenerates most often, which on Dream is the red--green
list.

\paragraph{Attacks and metrics.}
The attacks are word-level edits. They delete $10\%$ or $30\%$ of the words (del10,
del30), replace $30\%$ or $50\%$ of the words with WordNet synonyms (syn30,
syn50), or insert words drawn at random from the text itself, amounting to
$20\%$ of the text (ins20). The deletion and synonym attacks are adapted
from the MarkLLM toolkit \citep{pan2024markllm}. The
main text reports del30, syn30, and ins20. TPR@1\%FPR sets the threshold at
the $99$th percentile of the detector scores of the unwatermarked texts
listed in each caption. With fewer than $100$ such texts, this threshold
is close to their maximum score, and these texts cannot certify a
false-positive rate as low as $1\%$. Bootstrap intervals use $1000$
resamples of the generations.

\paragraph{How the defaults were chosen.}
The generation settings were chosen in two stages. First, the ablations of
\Cref{app:design}, run with one-sided enforcement at $\beta = 8$, fixed the
coloring. We use $q = 3$ because $q = 2$ detects less under edits and
larger $q$ costs perplexity (\Cref{tab:e1}). We use a single tap because
every extra tap exposes the checksum to more edits (\Cref{thm:robust},
\Cref{tab:e3}). We use $b = 0$ because it has the lowest null rate on
natural text under the default key (\Cref{app:residue-diag}).

Second, a selection experiment chose the enforcement side and the bias.
It used $48$ held-out C4 prompts (prompts $400$ to $447$, disjoint from the
$200$ evaluation prompts), and its selection rule was fixed before it ran.
The rule compared each setting with a reference, one-sided enforcement at
$\beta = 8$, which biases a position only when its left tap is unmasked.
A setting was eligible if its clean TPR was at least
$0.98$ and its del30 TPR was at least the reference's. On $48$ prompts the
first condition requires detecting every text, so $47$ of $48$, printed as
$0.98$, does not qualify. An eligible setting was then excluded if its
success at forgery without a language model (\Cref{tab:steal-pairs})
exceeded the reference's by more than $0.10$. Among the remaining settings,
the rule took the one with the lowest Qwen perplexity. It selected
either-side enforcement at $\beta = 5$ (Qwen perplexity $7.0$).
\Cref{tab:e9-screen} lists all $14$ settings, and \Cref{app:screen}
discusses them. On the $200$ evaluation prompts the selected
setting detects $68\%$ of deleted texts, against $79\%$ on the $48$
selection prompts, a gap within the noise of $48$ prompts. For this reason
\Cref{tab:e7} compares methods over the full range of $\beta$ and does not
rely on one operating point.

\begin{algorithm}[tbp]
\caption{One denoising step of TANGO generation. On Dream, which has no
CFG, the logits come from a plain forward pass. $a_\delta^{-1}$ is the
inverse of $a_\delta$ modulo $q$, which exists because $\gcd(a_\delta, q) =
1$. One-sided enforcement omits lines 5 and 6.}
\label{alg:bcc}
\begin{algorithmic}[1]
\Require key-derived coloring $\chi$, tap lag $\delta$, coefficient $a_\delta$ with $\gcd(a_\delta, q) = 1$, residue $b$; logit bias $\beta$; temperature $\tau$; unmasking schedule $K(\cdot)$
\State $\boldsymbol{\ell} \gets$ CFG logits for all masked positions
       \Comment{one bidirectional forward pass}
\For{each masked position $i$ in the step's candidate set}
  \If{$t_{i-\delta}$ is unmasked}
    \State $s_i \gets \big(b - a_\delta\, \chi(t_{i-\delta})\big) \bmod q$
           \Comment{left tap at lag $\delta$ unmasked}
  \ElsIf{$t_{i+\delta}$ is unmasked}
    \State $s_i \gets a_\delta^{-1}\big(b - \chi(t_{i+\delta})\big) \bmod q$
           \Comment{right neighbor at lag $\delta$ unmasked}
  \Else
    \State \textbf{continue} \Comment{neither unmasked, so no bias}
  \EndIf
  \State $\ell_{i,v} \gets \ell_{i,v} + \beta$ \textbf{ for all } $v$
         \textbf{ with } $\chi(v) = s_i$
\EndFor
\State sample one candidate per masked position at temperature $\tau$;
       unmask the $K$ most confident (schedule on remaining steps)
\end{algorithmic}
\end{algorithm}

\paragraph{Computational cost.}
Generation adds one logit update over the vocabulary per biased position per denoising step, an $O(V)$ cost next to the $O(Vd)$ output layer of the
forward pass, where $d$ is the embedding dimension. The coloring is built
once per key. Building it projects the $V$ token embeddings onto the key
direction at cost $O(Vd)$ and
sorts them at cost $O(V \log V)$. Detection then scores a text in one
$O(L)$ pass over its $L$ tokens.

\paragraph{Software, models, and data.}
The code runs on Python 3.12 with a CUDA 12 build of PyTorch and
Transformers at least 4.46 and below 4.57. Perplexity is scored with
Qwen/Qwen2.5-7B-Instruct and gpt2-large. The removal attack of
\Cref{app:results} rewrites words with distilroberta-base, and the
paraphrase translates with Helsinki-NLP/opus-mt-en-zh and opus-mt-zh-en.
The deletion and synonym attacks are adapted from MarkLLM, and synonyms
come from NLTK's WordNet. Prompts and natural texts come from the processed C4 file distributed with MarkLLM. The evaluation prompts are its first $200$ records, and the
Dream runs use the first $100$. The held-out selection prompts are records
$400$ to $447$, and the forging attacker of \Cref{tab:forge-fluent}
generates on records $200$ to $299$. Bootstrap intervals use seed $0$. The
published diffusion watermarks run from their public code at pinned
commits.

\FloatBarrier
\raggedbottom
\section{Calibration of the detector on unwatermarked text}
\label{sec:experiments-null}

The threshold $z_\star$ controls false positives only if the null of
\Cref{prop:null} describes unwatermarked text. That proposition assumes
that the scored token's color is uniform and independent of its tap's
color, and that matches at different positions are independent, which (B)
and (I) idealize. This appendix checks how far unwatermarked text for both models departs from these assumptions, and where the departure comes
from. It supports the per-model calibration of the threshold in
\Cref{sec:method-coloring,sec:method-detection}.

\subsection{LLaDA}
On LLaDA the fixed threshold $z > 4$ flags no unwatermarked text. We score
two kinds of unwatermarked text, $500$ texts each. Text of
uniformly random tokens satisfies the hypotheses of \Cref{prop:null} by
construction, and its score is close to standard normal (mean $-0.07$,
standard deviation $0.99$). In natural C4 text the colors of nearby tokens are
correlated (\Cref{app:residue-diag}). Its score is shifted left (mean
$-0.46$, standard deviation $1.08$), so the threshold flags
natural text less often than the normal tail predicts. With $0$ of $500$ texts above the threshold, the exact
two-sided $95\%$ Clopper--Pearson upper bound on the false-positive rate is
$0.74\%$ \citep{clopper1934}.

\FloatBarrier
\subsection{Dream}
On Dream the random-token null is again close to standard normal, but
natural text satisfies the checksum more often than chance. Its mean
score is $+1.07$, and $1.00\%$ of $300$ natural texts score above $4$
(\Cref{tab:e6-dream}). Because random tokens are calibrated, the excess
comes from structure in natural text. The likely source is the correlation
between the colors of neighboring tokens, which \Cref{app:residue-diag}
measures on LLaDA.

Building the coloring from whitened embeddings reduces the excess. ZCA
whitening \citep{kessy2018whitening} linearly transforms the embeddings so
that their dimensions are uncorrelated with unit variance, and we call the
coloring built from them the \emph{whitened coloring}. On Dream it lowers
the mean natural score from $1.07$ to $0.67$ and the rate above $z = 4$
from $1.00\%$ to $0.33\%$. Whitening barely changes detection. With one-sided enforcement at $\beta = 8$ ($n = 100$), del30 TPR
is $0.82$ with and without it, and syn30 and ins20 differ by $0.01$. We decide per model whether to whiten. The defaults on both models
do not whiten, so the Dream rows of \Cref{tab:headline,tab:e4-dream} use
the unwhitened coloring. Some LLaDA diagnostics in
\Cref{app:results,app:design} use the whitened coloring, as their captions
state. A new model needs its
threshold set on its own unwatermarked text.

\begin{table}[htbp]
\centering
\caption{Whitening lowers the score of unwatermarked natural text on Dream,
and random-token text stays calibrated with or without it. $300$ texts per
row. KS is the Kolmogorov--Smirnov
distance to $\mathcal{N}(0,1)$.}
\label{tab:e6-dream}
\small
\begin{tabular}{llrrrr}
\toprule
coloring & null source & mean $z$ & std & KS & FPR $z{>}4$ \\
\midrule
unwhitened & natural        & $1.07$ & 1.10 & 0.403 & 1.00\% \\
unwhitened & random tokens  & $0.02$ & 1.04 & 0.047 & 0.00\% \\
ZCA        & natural        & $0.67$ & 1.17 & 0.274 & 0.33\% \\
ZCA        & random tokens  & $-0.05$ & 1.00 & 0.057 & 0.00\% \\
\bottomrule
\end{tabular}
\end{table}

\FloatBarrier
\subsection{Correlated colors and the choice of residue}
\label{app:residue-diag}

Natural text departs from the hypotheses of \Cref{prop:null} in two ways,
and both make the null depend on the residue $b$, a choice the theory says
should not matter. First, the classes have equal numbers of tokens but not
equal frequencies (the class masses in \Cref{tab:residue-diag}). Second, tokens that occur near each
other have correlated embeddings and therefore correlated colors. Both
effects enter through the $q \times q$ table $P(u, w)$ of how often a tap
of color $u$ precedes a token of color $w$. A position matches residue $b$
when $w = (b - a_\delta u) \bmod q$, so each residue sums a different set
of $q$ cells,
\[
\Pr(M_i = 1) \;=\; \sum_{u \in \mathbb{Z}_q} P\big(u, (b - a_\delta u)
\bmod q\big) ,
\qquad
\text{which equals}\quad \sum_{u \in \mathbb{Z}_q} \pi_u\,
\pi_{(b - a_\delta u) \bmod q}
\]
when the two colors are independent with class masses $\pi_0, \dots,
\pi_{q-1}$. With unequal masses this rate already differs across
residues. For the natural-text masses of \Cref{tab:residue-diag} it is
$0.329$, $0.338$, and $0.333$ for $b = 0, 1, 2$. Correlated colors move
the rates further.

\Cref{tab:residue-diag} measures these rates for the default key on $2000$
natural C4 texts and on the $200$ unwatermarked LLaDA continuations of
\Cref{tab:headline}. On natural text at lag $1$, each cell of $P(u, w)$ differs from the product of its row and column frequencies by up to
$0.013$, and the null mean score ranges from $-0.93$ to $+0.98$ across
residues. At the default lag $2$ the difference is at most $0.005$ per
cell, and the range of null means is less than half as wide. This is why
the default tap skips the adjacent token.

The residue changes the null more than the watermarked score
(\Cref{tab:e2-lag2}). The three residues give mean watermarked scores
within $0.5$ of each other but null means that differ by up to $0.7$, and
TPR at a fixed false-positive rate depends on where the null sits. The
residue with the lowest null is a property of the key. Across eight fresh
keys, it was $b = 2$ for six, $b = 0$ for one, and $b = 1$ for one. The
residue should therefore be chosen, or the threshold calibrated, per key on
unwatermarked text. A tap-keyed residue, set by a keyed hash of the tap
color, avoids choosing $b$ and lowers the null mean to $-0.93$. It detects
slightly fewer unedited texts than $b = 0$ ($0.97$ against $1.00$) and more
deleted texts ($0.91$ against $0.79$).

\begin{table}[htbp]
\centering
\caption{Unwatermarked text satisfies each residue at a different rate
under the default coloring ($q = 3$). \emph{Class mass} is the fraction of the text's tokens in each color class, an estimate of the average of $m_s$ over positions. The token counts of the classes differ by at most one.
\emph{Null rate} is the fraction of scored positions whose color pair
satisfies the checksum for residue $b$ (chance $1/3$). \emph{Null mean $z$}
is the mean detector score per text. \emph{Natural} is $2000$ C4 texts, and
\emph{model} is the $200$ unwatermarked LLaDA continuations of
\Cref{tab:headline}.}
\label{tab:residue-diag}
\small
\begin{tabular}{llccc}
\toprule
text & lag & class mass ($0/1/2$) & null rate ($b = 0/1/2$) & null mean $z$ ($b = 0/1/2$) \\
\midrule
natural & 1 & $0.317/0.291/0.392$ & $0.302/0.366/0.332$ & $-0.93/{+0.98}/{-0.05}$ \\
natural & 2 & $0.317/0.291/0.392$ & $0.319/0.341/0.340$ & $-0.42/{+0.24}/{+0.19}$ \\
model   & 1 & $0.326/0.263/0.411$ & $0.298/0.362/0.340$ & $-0.83/{+0.68}/{+0.15}$ \\
model   & 2 & $0.326/0.263/0.411$ & $0.303/0.358/0.339$ & $-0.71/{+0.58}/{+0.13}$ \\
\bottomrule
\end{tabular}
\end{table}

\begin{table}[htbp]
\centering
\caption{The residues differ in the null score more than in the
watermarked score. LLaDA, lag $2$, $a_2 = 1$, one-sided enforcement at $\beta = 8$, $n = 100$. The clean and attack columns are TPR@1\%FPR with the $100$ unwatermarked continuations of the same prompts as negatives. \emph{Mean $z$} is the mean score of unedited watermarked
text, and \emph{null mean $z$} the mean score of $500$ natural texts (the first $500$ of the $2000$ texts of \Cref{tab:residue-diag}).
\emph{Tap-keyed} sets the residue by a keyed hash of the tap color.}
\label{tab:e2-lag2}
\small
\begin{tabular}{lrrrrrrrr}
\toprule
residue & AUROC & clean & mean $z$ & null mean $z$ & del30 & syn30 & ins20 & PPL GPT-2 \\
\midrule
$b = 0$ & 1.000 & 1.00 & 6.33 & $-0.46$ & 0.79 & 0.92 & 0.93 & 27.6 \\
$b = 1$ & 0.999 & 0.97 & 6.59 & $+0.24$ & 0.32 & 0.58 & 0.68 & 26.1 \\
$b = 2$ & 1.000 & 1.00 & 6.80 & $+0.22$ & 0.76 & 0.98 & 0.99 & 25.2 \\
tap-keyed & 0.991 & 0.97 & 6.02 & $-0.93$ & 0.91 & 0.93 & 0.97 & 26.9 \\
\bottomrule
\end{tabular}
\end{table}

\FloatBarrier
\section{Full results}
\label{app:results}

This appendix gives the full tables behind
\Cref{sec:experiments-headline,sec:experiments-stealth}, the Dream
results (\Cref{app:dream}), and the attacks and robustness checks that the main text only
mentions. Unless a
caption says otherwise, results are on LLaDA with the $200$ C4 evaluation
prompts, the attack columns are TPR@1\%FPR under del30, syn30, and ins20,
and the negatives for TPR@1\%FPR are the $200$ unwatermarked LLaDA
continuations of the same prompts (\Cref{app:protocol}). Several
experiments use \emph{one-sided} enforcement at $\beta = 8$, which biases a
position only when its left tap is unmasked. This is the reference setting
of the selection experiment (\Cref{app:protocol}). Some also use the
whitened coloring (\Cref{sec:experiments-null}). Each caption states its
configuration.

\FloatBarrier
\subsection{Key recovery and forgery from token freq{}uencies}

The frequency attacker of \Cref{sec:experiments-stealth} ranks single tokens by how much
more often they appear in watermarked than in unwatermarked text.
\Cref{tab:stealth} reports how much the token frequencies shift, how well
the shift predicts the key, and how often forgeries built from the
predicted key pass the detector. Against TANGO the predicted key is at
chance and no forgery passes. Against the red--green list the key is
recovered with AUC $0.817$ and every forgery passes. The table uses
one-sided enforcement at $\beta = 8$. On the either-side texts of
\Cref{tab:headline}, key recovery against TANGO is also $0.504$.
\Cref{fig:steal} shows that the red--green list's key becomes easier to
recover as the attacker collects more texts, while TANGO's stays at chance,
consistent with \Cref{thm:unforgeable}.

\begin{table}[htbp]
\centering
\caption{Token frequencies reveal the red--green list's key but not
TANGO's. $200$ texts per method, TANGO with one-sided enforcement at $\beta
= 8$, red--green list at bias $5$. \emph{Frequency shift} is the
total-variation distance between watermarked and unwatermarked token
frequencies (sampling noise alone gives $0.210$). \emph{Key recovery} is
the AUC of predicting the coloring from per-token frequency shifts.
\emph{Forgery success} is the fraction of texts generated from the
predicted key that the detector accepts at $z > 4$.}
\label{tab:stealth}
\small
\begin{tabular}{@{}lrrr@{}}
\toprule
method & frequency shift (TV) & key recovery AUC & forgery success \\
\midrule
TANGO      & \textbf{0.217} & \textbf{0.504} & \textbf{0.00} \\
Red--green list & 0.340 & 0.817 & 1.00 \\
\bottomrule
\end{tabular}
\end{table}

\begin{figure}[htbp]
\centering
\begin{tikzpicture}
\begin{axis}[
  width=0.72\linewidth, height=0.48\linewidth,
  font=\small,
  xlabel={number of texts $N$ the attacker holds},
  ylabel={key-recovery AUC},
  xmode=log, log basis x=10,
  xmin=22, xmax=110, ymin=0.4, ymax=1.0,
  xtick={25,50,100},
  xticklabels={$25$, $50$, $100$},
  ytick={0.4,0.5,0.6,0.7,0.8,0.9,1.0},
  grid=both,
  grid style={line width=0.3pt, draw=gray!18},
  major grid style={line width=0.4pt, draw=gray!28},
  axis lines=left, tick align=outside,
  legend style={at={(0.03,0.97)}, anchor=north west,
                font=\footnotesize, draw=gray!40, rounded corners=2pt,
                fill=white, fill opacity=0.9, text opacity=1},
  legend cell align=left,
  every axis plot/.append style={line width=1.1pt},
]
\addplot[refline, densely dotted, line width=1pt, forget plot]
    coordinates {(22,0.5)(110,0.5)};
\addlegendimage{refline, densely dotted, line width=1pt}
\addplot[color=baseline, mark=square*, mark size=2.4pt,
         mark options={fill=baseline, draw=baseline}]
    coordinates {(25,0.738)(50,0.782)(100,0.806)};
\addplot[tango, mark=*, mark size=2.4pt,
         mark options={fill=tango, draw=tango}]
    coordinates {(25,0.499)(50,0.499)(100,0.499)};
\legend{chance (0.5), Red--green list, TANGO}
\end{axis}
\end{tikzpicture}
\caption{Key recovery from token frequencies improves with the number of
texts against the red--green list ($0.738$ at $25$ texts, $0.806$ at $100$)
and stays at chance against TANGO. Subsets of a pool of $100$ texts per
method, TANGO with one-sided enforcement at $\beta = 6$ and a whitened
coloring, red--green list at bias $5$.}
\label{fig:steal}
\end{figure}
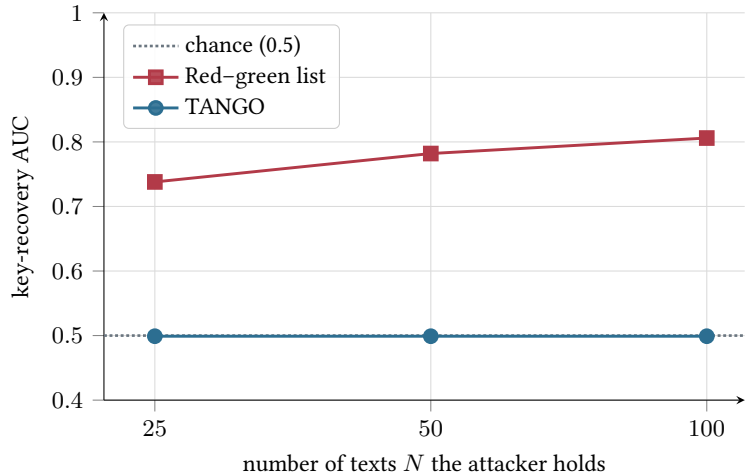

\FloatBarrier
\subsection{Key recovery and forgery from token pairs}

\Cref{thm:unforgeable} covers single-token frequencies only. This
subsection measures an attacker it does not cover, one who counts token
pairs at the tap lag. \Cref{tab:steal-pairs} reports two steps of this
attack for every method. The first step ranks pairs by how much more often
they appear in watermarked than in unwatermarked text. Its \emph{pair AUC}
measures how well this ranking predicts which pairs satisfy the checksum
under the true key. The second step builds a text without a language model,
by chaining pairs that are over-represented in watermarked text, and
submits it to the detector.
Against TANGO the pair AUC is $0.62$, and $0.78$ on pairs seen at least
three times. At $200$ texts these forgeries pass TANGO's detector $74\%$ of
the time, against $96\%$ for the red--green list. Against dgMARK at its greedy defaults, at most one of $250$ forgeries passes at any $N$ (mean score $1.4$ at $200$ texts, in one of the five repetitions), because
those watermarked texts carry only a small excess of even token ids to
learn from (\Cref{app:results-external}).

Text strung together without a language model is not fluent.
\Cref{tab:forge-fluent} therefore reports a more realistic attacker, who has the
model but not the key. While generating, it adds a logit bias $\beta_{\mathrm{forge}}$,
the forging bias. Against TANGO the bias favors the tokens that followed
the tap token in the collected pairs, and against the red--green list it
favors the recovered green list. At $\beta_{\mathrm{forge}} = 4$, its texts
pass TANGO's detector $4\%$ of the time and the red--green detector $51\%$
of the time, at GPT-2-large perplexity $13.5$ and $15.2$ against $13.1$ for
unwatermarked text. At $\beta_{\mathrm{forge}} = 8$ the
rates are $55\%$ and $98\%$, and the TANGO forgeries have perplexity
$21.1$, against $13.1$ for unwatermarked text.

\begin{table}[htbp]
\centering
\caption{Forgeries built from token-pair counts pass TANGO's detector
less often than the red--green list's at every $N$. Each cell is
\emph{pair AUC / forgery success}, averaged over five random subsets of
the $N$ texts the attacker holds. Pair AUC uses each method's tap lag ($2$
for TANGO, $1$ otherwise). Forgery success is the fraction of $50$ forged
$128$-token texts accepted at $z > 4$. TANGO uses the either-side texts of
\Cref{tab:headline}. A dash (\na{}) marks dgMARK's pair AUC, which its keyless test leaves undefined (\Cref{app:results-external}), and forgeries we did not score, because the DLM watermark and KGW detectors run in their own code base. Dream has $100$ texts per method, so it has no $N = 200$ entry. KGW is the context-hashed green list of \citet{kirchenbauer2023watermark}, run through the DLM watermark's code. dgMARK runs without its beam-search lookahead (\Cref{app:results-external}).}
\label{tab:steal-pairs}
\small
\begin{tabular}{llcccc}
\toprule
model & method & $N=25$ & $N=50$ & $N=100$ & $N=200$ \\
\midrule
LLaDA & TANGO & 0.62 / 0.12 & 0.62 / 0.18 & 0.62 / 0.48 & 0.62 / 0.74 \\
LLaDA & Red--green list & 0.63 / 0.51 & 0.64 / 0.74 & 0.65 / 0.89 & 0.66 / 0.96 \\
LLaDA & Gumbel & 0.53 / 0.17 & 0.54 / 0.25 & 0.53 / 0.24 & 0.53 / 0.56 \\
LLaDA & DLM watermark, defaults & 0.54 / \na & 0.55 / \na & 0.54 / \na & 0.54 / \na \\
LLaDA & DLM watermark, matched, bias $2$ & 0.53 / \na & 0.53 / \na & 0.55 / \na & 0.55 / \na \\
LLaDA & DLM watermark, matched, bias $3$ & 0.55 / \na & 0.55 / \na & 0.56 / \na & 0.57 / \na \\
LLaDA & DLM watermark, matched, bias $4$ & 0.57 / \na & 0.58 / \na & 0.59 / \na & 0.60 / \na \\
LLaDA & KGW (context-hashed), matched, bias $2$ & 0.55 / \na & 0.55 / \na & 0.56 / \na & 0.56 / \na \\
LLaDA & dgMARK, greedy defaults & \na{} / 0.00 & \na{} / 0.00 & \na{} / 0.00 & \na{} / 0.00 \\
LLaDA & dgMARK, matched, top-$k$ $3$ & \na{} / 0.09 & \na{} / 0.21 & \na{} / 0.26 & \na{} / 0.44 \\
\midrule
Dream & TANGO & 0.58 / 0.14 & 0.59 / 0.31 & 0.59 / 0.49 & \na \\
Dream & Red--green list & 0.69 / 0.29 & 0.67 / 0.51 & 0.68 / 0.64 & \na \\
\bottomrule
\end{tabular}
\end{table}

\begin{table}[htbp]
\centering
\caption{At forging bias $\beta_{\mathrm{forge}} = 4$, an attacker with the model but not
the key passes TANGO's detector in $4\%$ of cases and the red--green
list's in $51\%$. The attacker mines the $200$ watermarked texts of each
method from \Cref{tab:headline} and generates on $100$ unseen prompts.
Against TANGO it biases each position toward the tokens that, in the
collected lag-$2$ pairs, followed the token two positions to its left.
Against the red--green list it biases every position toward the tokens
most over-represented in watermarked text. Success is at each detector's empirical $1\%$ FPR
threshold on these prompts ($z = 1.96$ for TANGO, $4.15$ for the
red--green list) and at the fixed $z > 4$. PPL is GPT-2-large perplexity,
and unwatermarked text on these prompts has PPL $13.1$.}
\label{tab:forge-fluent}
\small
\begin{tabular}{llrrrr}
\toprule
$\beta_{\mathrm{forge}}$ & detector & mean $z$ & success (1\% FPR) & success ($z > 4$) & PPL GPT-2 \\
\midrule
4 & TANGO & 0.08 & 0.04 & 0.00 & 13.5 \\
4 & Red--green list & 4.15 & 0.51 & 0.56 & 15.2 \\
8 & TANGO & 2.62 & 0.55 & 0.23 & 21.1 \\
8 & Red--green list & 10.56 & 0.98 & 0.98 & 30.2 \\
\bottomrule
\end{tabular}
\end{table}

\FloatBarrier
\subsection{The published diffusion watermarks}
\label{app:results-external}

Neither of the two published watermarks for masked-diffusion models that
we ran detects more than TANGO or the red--green list at a similar
perplexity cost (\Cref{tab:external}), which supports the comparison in
\Cref{sec:experiments-headline}. We ran the DLM watermark
\citep{wmdlm2025} and dgMARK \citep{dgmark2026} from their public code on
the same $200$ prompts, at their own default decoding settings and at
ours. We ran dgMARK without the one-step lookahead beam search of its
Section~3.4, which its generation script runs by default, and with its default $32$-token blocks in both runs. Their
samplers differ from ours, so each perplexity is paired with an
unwatermarked control from the same sampler. The best-detecting settings
of the published watermarks are the DLM watermark at bias $4$, which detects $64\%$ of clean
texts at about $1.14$ times its control perplexity, and dgMARK with
top-$k$ sampling, which detects $84\%$ at $2.2$ times. The DLM watermark
falls on the red--green list's curve in \Cref{tab:e7}. That list detects
$58\%$ of clean texts at $1.1$ times its control perplexity (bias $2$) and
$83\%$ at $1.2$ times (bias $3$). dgMARK detects fewer clean texts than TANGO, which detects
$97\%$ at $1.7$ times.

The released dgMARK detector needs no key. When a private key is set, as
in our runs, its code derives the expected
parity and the observed token parity from the same key bit, so the two
cancel, and the test reduces to counting even token ids. Any text rich in
even token ids therefore passes it. The watermarked texts have $58\%$
(greedy defaults) and $68\%$ (top-$k$) even ids, against $48\%$ for
unwatermarked text, which is why its forgeries in \Cref{tab:steal-pairs}
target even ids. We report dgMARK as released. Its authors note that a keyed pseudorandom
function can replace the token-id parity.

\begin{table}[htbp]
\centering
\caption{Neither of the two published diffusion watermarks we ran detects
more than TANGO or the red--green list at a similar perplexity increase
over its own control. $200$ C4 prompts, $128$ tokens, LLaDA-8B. TANGO, the red--green
list, and the Gumbel rule are at the settings of \Cref{tab:headline}. The
published diffusion watermarks were run from their public code at their own default decoding settings and at our decoding settings
(\emph{matched}), dgMARK without its beam-search lookahead. The DLM watermark's default run and both dgMARK runs, with their controls, decode in $32$-token blocks, and all other runs decode the $128$ tokens as one block. A dash marks a value that is undefined or that we did not compute, namely frequency AUC for the Gumbel rule, the DLM watermark, and KGW, pair AUC for dgMARK's keyless test, and forgery success for the DLM watermark and KGW, whose forgeries we did not score. Each perplexity is followed, in parentheses, by that of
unwatermarked text from the same sampler, and both are over all
continuations. \emph{Freq.\ AUC} is key recovery from token frequencies
(\Cref{tab:stealth}). \emph{Pair AUC} and \emph{forge} are at $N = 200$ from \Cref{tab:steal-pairs}, where KGW is also defined. In the DLM watermark's default row, $2$ or $3$ of the $200$ texts per column received no score from its detector and are left out.}
\label{tab:external}
\footnotesize
\setlength{\tabcolsep}{3pt}
\resizebox{\linewidth}{!}{\begin{tabular}{@{}lrccccrrccc@{}}
\toprule
 & & \multicolumn{4}{c}{TPR@1\%FPR} & \multicolumn{2}{c}{PPL (control)} & & & \\
\cmidrule(lr){3-6}\cmidrule(lr){7-8}
method & AUROC & clean & del30 & syn30 & ins20 & Qwen & GPT-2 & freq.\ AUC & pair AUC & forge \\
\midrule
TANGO ($\beta = 5$, either side) & 0.998 & 0.97 & 0.68 & 0.92 & 0.91 & 7.0 (4.1) & 18.9 (12.9) & 0.50 & 0.62 & 0.74 \\
Red--green list, bias $5$ & 0.993 & 0.98 & 0.96 & 0.94 & 0.98 & 7.4 (4.1) & 17.6 (12.9) & 0.82 & 0.66 & 0.96 \\
Gumbel rule & 0.815 & 0.15 & 0.10 & 0.07 & 0.07 & 4.5 (4.1) & 13.6 (12.9) & \na & 0.53 & 0.56 \\
\midrule
DLM watermark, defaults (bias $2$) & 0.843 & 0.37 & 0.30 & 0.22 & 0.24 & 6.5 (6.1) & 39.9 (71.1) & \na & 0.54 & \na \\
DLM watermark, matched, bias $2$ & 0.808 & 0.19 & 0.16 & 0.09 & 0.09 & 4.2 (4.0) & 24.1 (21.5) & \na & 0.55 & \na \\
DLM watermark, matched, bias $3$ & 0.873 & 0.34 & 0.20 & 0.16 & 0.14 & 4.2 (4.0) & 24.3 (21.5) & \na & 0.57 & \na \\
DLM watermark, matched, bias $4$ & 0.915 & 0.64 & 0.44 & 0.32 & 0.39 & 4.6 (4.0) & 22.9 (21.5) & \na & 0.60 & \na \\
KGW (context-hashed), matched, bias $2$ & 0.745 & 0.17 & 0.11 & 0.07 & 0.12 & 4.4 (4.0) & 23.2 (21.5) & \na & 0.56 & \na \\
dgMARK, greedy defaults & 0.871 & 0.41 & 0.43 & 0.40 & 0.41 & 3.7 (3.3) & 23.2 (59.4) & 0.59 & \na & 0.00 \\
dgMARK, matched, top-$k$ $3$ & 0.983 & 0.84 & 0.78 & 0.64 & 0.78 & 8.9 (4.0) & 34.2 (31.3) & 0.70 & \na & 0.44 \\
\bottomrule
\end{tabular}}
\end{table}

\FloatBarrier
\subsection{Detection against perplexity at every strength}

This subsection expands the comparison of
\Cref{sec:experiments-headline}. A single operating point per method can
favor whichever method happens to sit at a lower perplexity. \Cref{tab:e7} therefore varies each method's
strength under identical decoding, so that methods can be compared at
matched perplexity. At similar Qwen perplexity the red--green list detects
more deleted texts than TANGO. At perplexity $7.6$ against TANGO's $7.2$ it
detects $96\%$ against $68\%$, and at $9.9$ against $8.4$ it detects $98\%$
against $81\%$.

\begin{table}[htbp]
\centering
\caption{The red--green list detects more deleted texts than TANGO at
matched perplexity. LLaDA, $n = 200$ per setting, with the $200$
unwatermarked continuations as negatives. TANGO (either side) varies
$\beta$, the red--green list its bias, and the Gumbel rule its sampling
temperature. Perplexity is over non-degenerate continuations, which are at
least $94\%$ of every setting (\Cref{app:protocol}). The Gumbel row at
$\tau = 1.0$ is a second generation run of the setting in
\Cref{tab:headline}, so it differs from that row by sampling noise.}
\label{tab:e7}
\small
\begin{tabular}{lrrrrrrr}
\toprule
setting & AUROC & clean & del30 & syn30 & ins20 & PPL GPT-2 & PPL Qwen \\
\midrule
TANGO $\beta = 4$ & 0.998 & 0.97 & 0.61 & 0.84 & 0.85 & 16.8 & 6.1 \\
TANGO $\beta = 5$ (\Cref{tab:headline}) & 0.998 & 0.97 & 0.68 & 0.92 & 0.91 & 19.4 & 7.2 \\
TANGO $\beta = 6$ & 0.999 & 0.99 & 0.81 & 0.95 & 0.96 & 21.3 & 8.4 \\
TANGO $\beta = 8$ & 1.000 & 1.00 & 0.93 & 0.99 & 0.98 & 31.4 & 13.5 \\
TANGO $\beta = 10$ & 1.000 & 1.00 & 0.97 & 0.98 & 0.99 & 51.8 & 23.4 \\
\midrule
Red--green list, bias $2$ & 0.940 & 0.58 & 0.38 & 0.39 & 0.69 & 13.5 & 4.6 \\
Red--green list, bias $3$ & 0.979 & 0.83 & 0.71 & 0.71 & 0.90 & 13.9 & 5.2 \\
Red--green list, bias $4$ & 0.995 & 0.94 & 0.85 & 0.85 & 0.96 & 16.3 & 6.3 \\
Red--green list, bias $5$ (\Cref{tab:headline}) & 0.993 & 0.98 & 0.96 & 0.94 & 0.98 & 18.0 & 7.6 \\
Red--green list, bias $6$ & 0.999 & 0.98 & 0.98 & 0.98 & 0.98 & 21.6 & 9.9 \\
\midrule
Gumbel $\tau = 1.0$ & 0.821 & 0.13 & 0.06 & 0.07 & 0.10 & 14.0 & 4.6 \\
Gumbel $\tau = 1.2$ & 0.927 & 0.48 & 0.23 & 0.32 & 0.36 & 18.1 & 6.2 \\
Gumbel $\tau = 1.5$ & 0.999 & 0.97 & 0.83 & 0.90 & 0.96 & 49.8 & 20.6 \\
\bottomrule
\end{tabular}
\end{table}

\FloatBarrier
\subsection{Results on Dream}
\label{app:dream}

This subsection expands the Dream rows of \Cref{tab:headline}. On Dream
the red--green list degrades text more than TANGO. \Cref{tab:e4-dream}
gives the Dream rows with their intervals and GPT-2 perplexity. The
red--green list collapses $62\%$ of its texts into a repeated phrase,
against $20\%$ of unwatermarked texts and $31\%$ of TANGO texts
(\Cref{app:protocol}). On the texts that do not collapse, its perplexity
is also above TANGO's under both scorers. The collapse does not hurt its
detection. Its mean score is $12.2$ on the non-degenerate texts and
$14.3$ on the degenerate ones. The Dream null is in
\Cref{sec:experiments-null}, and Dream forgery from token pairs is in
\Cref{tab:steal-pairs}.

Key recovery from token frequencies stays at chance on Dream, as on LLaDA.
With one-sided enforcement at $\beta = 8$ ($100$ texts), the total-variation distance between watermarked and unwatermarked token frequencies is $0.319$ for TANGO, above the $0.229$ that sampling noise alone produces between two unwatermarked samples, and the red--green list reaches $0.610$. In the same run, the key-recovery AUC is $0.503$ against TANGO and $0.798$ against the red--green list, and
on the either-side texts of \Cref{tab:headline} it is $0.501$ against TANGO. In the one-sided run, no forgery built from the predicted key passes TANGO's detector,
while every one passes the red--green detector.

\begin{table}[htbp]
\centering
\caption{On Dream, TANGO and the red--green list detect every unedited
text, and TANGO has the lower perplexity on non-degenerate continuations.
$n = 100$ per method, unwhitened coloring, with the $100$ unwatermarked
continuations as negatives and $95\%$ bootstrap intervals. PPL is
GPT-2-large perplexity over non-degenerate continuations, and
\emph{degenerate} is the fraction of continuations that collapsed into a
repeated phrase (\Cref{app:protocol}).}
\label{tab:e4-dream}
\small
\resizebox{\linewidth}{!}{\begin{tabular}{lrrrrrrr}
\toprule
 & & \multicolumn{4}{c}{TPR@1\%FPR} & & \\
\cmidrule(lr){3-6}
method & AUROC & clean & del30 & syn30 & ins20 & PPL GPT-2 & degenerate \\
\midrule
TANGO ($\beta = 5$, either side) & 1.000 & 1.00 [1.00,1.00] & 0.73 [0.65,0.81] & 0.98 [0.95,1.00] & 0.97 [0.93,1.00] & 15.2 & 31\% \\
Red--green list & 1.000 & 1.00 [1.00,1.00] & 0.98 [0.95,1.00] & 0.98 [0.95,1.00] & 0.98 [0.95,1.00] & 18.1 & 62\% \\
unwatermarked & \na & \na & \na & \na & \na & 9.9 & 20\% \\
\bottomrule
\end{tabular}}
\end{table}

\FloatBarrier
\subsection{Removal by an attacker who holds the key}

A provider whose key leaks should know how much editing removes the
watermark. \Cref{tab:b3} compares two attackers who regenerate the same
number of words. Because the attacker edits words and the checksum reads
tokens, the key-aware attacker approximates each scored pair by the first
tokens of two words that are two words apart. It uses the key to find the
word pairs that satisfy the checksum and regenerates a random subset of
those words. The random attacker regenerates random words. Both replace
each chosen word with a prediction of a masked language model
(distilroberta-base) that differs from the original word. The setting is
one-sided enforcement at $\beta = 6$ with the whitened coloring ($n =
60$). Unedited texts have mean $z = 5.55$ at GPT-2 perplexity $19.2$.
At a $10\%$ budget the two attackers lower the score by about the same
amount (mean $z$ $4.34$ against $4.42$). At $25\%$ and $50\%$ the
key-aware attacker lowers it less than the random one, so for this
attacker the key does not make removal easier. At a $25\%$ budget
about a quarter of the texts remain above $z = 4$, and at $50\%$ almost
none do. The key-aware attacker chooses which words to regenerate but not
their replacements, and it locates taps at the word level while TANGO's tap
is two tokens back. The table therefore describes only attackers of this
kind.

\begin{table}[htbp]
\centering
\caption{Knowing the key does not help this attacker remove the
watermark. One-sided enforcement at $\beta = 6$, whitened coloring, $n =
60$, clean mean $z = 5.55$. Each budget is the fraction of words
regenerated. \emph{Det.\ rate} is the fraction of texts still above $z =
4$. PPL is GPT-2-large perplexity.}
\label{tab:b3}
\small
\begin{tabular}{lrrrrrr}
\toprule
 & \multicolumn{3}{c}{key-aware} & \multicolumn{3}{c}{random} \\
\cmidrule(lr){2-4}\cmidrule(lr){5-7}
budget & mean $z$ & PPL GPT-2 & det.\ rate & mean $z$ & PPL GPT-2 & det.\ rate \\
\midrule
10\% & 4.34 & 28.1 & 0.62 & 4.42 & 28.5 & 0.60 \\
25\% & 3.11 & 44.3 & 0.27 & 2.81 & 44.2 & 0.20 \\
50\% & 1.54 & 76.8 & 0.02 & 1.03 & 95.2 & 0.02 \\
\bottomrule
\end{tabular}
\end{table}

\FloatBarrier
\subsection{Paraphrase}
\label{app:paraphrase}

\Cref{thm:robust} covers edits that keep positions aligned. A paraphrase reorders whole clauses, so many tokens no longer sit at lag $2$ from their original taps. We
paraphrase by back-translation, from English to Chinese and back with the
Helsinki-NLP Marian models ($n = 100$, one-sided enforcement at $\beta = 8$, with the $100$ unwatermarked continuations of the same prompts as negatives). TANGO's TPR@1\%FPR falls from $1.00$ to $0.34$ (mean $z$ from $6.5$ to
$0.9$). The red--green list falls to $0.61$ (mean $z$ from $9.0$ to $4.5$),
because its signal sits in single tokens and survives reordering. Neither
watermark is robust to paraphrase.

\FloatBarrier
\subsection{Variance over seeds and keys}
\label{app:seeds}

Single runs hide the variation due to the sampling seed and the secret
key. Both checks below use one-sided settings, not the setting of
\Cref{tab:headline}. \Cref{tab:e6ci} repeats one-sided enforcement at
$\beta = 8$ over three generation seeds ($32$ texts per seed). The per-seed
standard deviation of TPR is $0.04$ without edits and up to $0.19$ under
the heaviest edits. \Cref{tab:b5} repeats one-sided enforcement at $\beta =
6$ with a whitened coloring over four independent keys ($40$ texts per
key). Clean AUROC is $0.998 \pm 0.002$ across keys, so detection does not
depend on a lucky key direction.

\begin{table}[htbp]
\centering
\caption{Detection over three generation seeds, one-sided enforcement at
$\beta = 8$ with a single tap at lag $2$ and the unwhitened coloring, $32$ texts per seed, with the $96$ unwatermarked continuations
as negatives. Intervals are $95\%$ percentile intervals from $1000$
bootstrap resamples of the pooled texts, and the per-seed std is the population standard deviation over the three seeds.}
\label{tab:e6ci}
\small
\begin{tabular}{lrr}
\toprule
attack & TPR@1\%FPR (95\% CI) & per-seed std \\
\midrule
clean      & 0.98 [0.95,1.00] & 0.04 \\
del10 & 0.92 [0.85,0.97] & 0.08 \\
del30 & 0.62 [0.52,0.72] & 0.18 \\
syn30 & 0.85 [0.78,0.92] & 0.11 \\
syn50 & 0.61 [0.51,0.71] & 0.19 \\
ins20 & 0.91 [0.84,0.96] & 0.09 \\
\bottomrule
\end{tabular}
\end{table}

\begin{table}[htbp]
\centering
\caption{Detection over four independent secret keys, one-sided
enforcement at $\beta = 6$, whitened coloring, $40$ texts per key, with the $40$ unwatermarked continuations of the same prompts as negatives. Each key uses its own residue, $b = 2$, $0$, $2$, and $0$. The std is the population standard deviation over the four keys. The
second row is the mean score of $300$ natural C4 texts.}
\label{tab:b5}
\small
\begin{tabular}{lrr}
\toprule
quantity & mean & std \\
\midrule
clean AUROC & 0.998 & 0.002 \\
mean $z$ on natural text & $-0.11$ & 0.28 \\
\bottomrule
\end{tabular}
\end{table}

\FloatBarrier
\section{Full ablations}
\label{app:design}

This appendix supports the design choices of
\Cref{sec:experiments-sweeps}. It gives every column of the ablations
condensed into \Cref{tab:ablations}, the selection experiment that chose the enforcement
side and bias, and a check of the length scaling of \Cref{prop:ez}. The
residue ablation is in \Cref{app:residue-diag}, next to the calibration
diagnostic that explains it. Unless a caption says otherwise, the
ablations use one-sided enforcement at $\beta = 8$ and residue $b = 0$, PPL
is GPT-2-large perplexity over all continuations, and null FPR is the
fraction of $200$ natural C4 texts with $z > 4$.

The ablations of \Cref{tab:e1,tab:e3} are small. Each cell has $24$
watermarked texts, and the TPR@1\%FPR threshold is set on the $24$
unwatermarked continuations of the same prompts, so it is close to their
maximum score. A cell at TPR $1.00$ has an exact two-sided $95\%$
Clopper--Pearson lower bound of $0.86$. The noise is also visible directly.
The $q = 3$ row of \Cref{tab:e1} and the $\{1\}$ row of \Cref{tab:e3} are
the same configuration run twice, and their TPRs differ by up to $0.13$.
Differences of that size within a table are therefore not significant, and
we draw conclusions only from trends across several rows.

\FloatBarrier
\subsection{Number of colors}
\label{app:sweeps}

The semantic-quantile coloring (\emph{quantile} for short) keeps the null
FPR at or below $0.5\%$ for every number of colors, while the
semantic-cluster coloring (\emph{cluster}) of \Cref{sec:method-coloring}
does not (\Cref{tab:e1}). The cluster coloring flags $60\%$ of
natural texts at $q = 2$ and $30.5\%$ at $q = 3$, because tokens with similar
embeddings, which often occur near each other, share a cluster, so the checksum holds by chance far more
often than $1/q$. For the quantile coloring, $q = 2$ detects less under
edits than $q \geq 3$, and perplexity grows from $22.2$ at $q = 3$ to about
$30$ at $q \geq 8$.

\begin{table}[htbp]
\centering
\caption{The semantic-quantile coloring keeps the null FPR at or below
$0.5\%$ for every $q$, while the cluster coloring flags $60\%$ of natural
texts at $q = 2$ and $30.5\%$ at $q = 3$. Lag $1$, $n = 24$ per row. The residue is the default key's, $b = 0$ for $q \leq 3$ and $b = 3$, $7$, and $15$ for $q = 4$, $8$, and $16$.}
\label{tab:e1}
\small
\resizebox{\linewidth}{!}{\begin{tabular}{lrrrrrrrrr}
\toprule
coloring & AUROC & clean & del10 & del30 & syn30 & syn50 & ins20 & PPL GPT-2 & null FPR (\%) \\
\midrule
quantile, $q = 2$ & 0.989 & 0.92 & 0.79 & 0.62 & 0.71 & 0.58 & 0.83 & 20.4 & 0.0 \\
quantile, $q = 3$ & 1.000 & 1.00 & 1.00 & 0.92 & 0.96 & 0.75 & 1.00 & 22.2 & 0.0 \\
quantile, $q = 4$ & 1.000 & 1.00 & 0.96 & 0.79 & 0.83 & 0.67 & 0.96 & 26.8 & 0.5 \\
quantile, $q = 8$ & 1.000 & 1.00 & 1.00 & 0.83 & 0.83 & 0.62 & 1.00 & 29.9 & 0.5 \\
quantile, $q = 16$ & 0.998 & 0.96 & 0.96 & 0.88 & 0.96 & 0.92 & 0.96 & 29.5 & 0.0 \\
cluster, $q = 2$ & 0.817 & 0.38 & 0.25 & 0.12 & 0.17 & 0.08 & 0.46 & 13.3 & 60.0 \\
cluster, $q = 3$ & 0.916 & 0.71 & 0.58 & 0.46 & 0.54 & 0.50 & 0.54 & 17.5 & 30.5 \\
cluster, $q = 4$ & 0.972 & 0.92 & 0.92 & 0.92 & 0.83 & 0.75 & 0.92 & 23.3 & 0.0 \\
cluster, $q = 8$ & 0.930 & 0.67 & 0.75 & 0.67 & 0.83 & 0.79 & 0.62 & 24.0 & 0.0 \\
cluster, $q = 16$ & 0.975 & 0.79 & 0.75 & 0.58 & 0.50 & 0.42 & 0.75 & 20.3 & 0.5 \\
\bottomrule
\end{tabular}}
\end{table}

\FloatBarrier
\subsection{Tap set}

Detection under edits falls as taps are added (\Cref{tab:e3,fig:taps}).
This is consistent with \Cref{thm:robust}, by which each added tap
multiplies the expected score under random substitutions by another factor $1 - \rho$. A
single tap detects best under every edit, and with four taps del30 TPR
falls from $0.92$ to $0.12$. The trend is not monotone in every column,
but the exceptions are within the noise of $24$ texts per row. The theorem
depends only on $|\mathcal{D}|$, so it predicts the same robustness for lag
$1$ and lag $2$. In \Cref{tab:e3}, lag $2$ detects as well as or better than lag $1$ in every column. We use lag $2$ because the adjacent token's color is the most correlated with the scored
token's color in natural text (\Cref{app:residue-diag}), and because the
lag-$1$ tap leaks more to the pair-counting attacker of \Cref{tab:steal-pairs}
(\Cref{tab:e9-screen}).

\begin{table}[htbp]
\centering
\caption{Adding taps lowers detection under edits. $q = 3$
semantic-quantile coloring, $n = 24$ per row.}
\label{tab:e3}
\small
\begin{tabular}{lrrrrrrrrr}
\toprule
taps $\mathcal{D}$ & AUROC & clean & del10 & del30 & syn30 & syn50 & ins20 & PPL GPT-2 & null FPR (\%) \\
\midrule
$\{1\}$ & 0.990 & 0.96 & 0.96 & 0.92 & 0.83 & 0.71 & 0.96 & 25.6 & 0.0 \\
$\{2\}$ & 1.000 & 1.00 & 1.00 & 0.92 & 0.96 & 0.92 & 0.96 & 25.0 & 0.0 \\
$\{1, 2\}$ & 0.990 & 0.96 & 0.79 & 0.25 & 0.58 & 0.17 & 0.79 & 19.6 & 0.0 \\
$\{1, 3\}$ & 1.000 & 1.00 & 0.88 & 0.50 & 0.62 & 0.62 & 0.75 & 18.5 & 0.0 \\
$\{1, 2, 3\}$ & 0.986 & 0.83 & 0.62 & 0.25 & 0.25 & 0.04 & 0.54 & 16.7 & 0.0 \\
$\{1, 2, 3, 4\}$ & 0.984 & 0.83 & 0.50 & 0.12 & 0.38 & 0.25 & 0.42 & 18.9 & 0.0 \\
\bottomrule
\end{tabular}
\end{table}

\begin{figure}[htbp]
\centering
\begin{tikzpicture}
\begin{axis}[
  width=0.55\linewidth, height=0.36\linewidth,
  font=\footnotesize,
  xlabel={number of taps $|\mathcal{D}|$},
  ylabel={TPR@1\%FPR},
  xmin=0.7, xmax=5.0, ymin=0, ymax=1.08,
  xtick={1,2,3,4}, ytick={0,0.5,1},
  yticklabels={$0$,$0.5$,$1$},
  axis lines=left, axis line style={mut, thin},
  tick style={mut, thin}, tick align=outside,
  ymajorgrids, grid style={gray!18, thin},
  clip=false,
]
\addplot[refline, dashed, line width=0.8pt, domain=1:4, samples=60] {0.92*pow(0.70,x-1)};
\addplot[tango, line width=0.9pt, mark=*, mark size=2.1pt] coordinates {
  (1,0.92) (2,0.25) (3,0.25) (4,0.12)};
\addplot[attackb, line width=0.9pt, mark=square*, mark size=1.9pt] coordinates {
  (1,0.96) (2,0.58) (3,0.25) (4,0.38)};
\addplot[attackc, line width=0.9pt, mark=triangle*, mark size=2.4pt] coordinates {
  (1,0.96) (2,0.79) (3,0.54) (4,0.42)};
\draw[gray!55, thin] (axis cs:4.06,0.42) -- (axis cs:4.30,0.50);
\draw[gray!55, thin] (axis cs:4.06,0.38) -- (axis cs:4.30,0.36);
\draw[gray!55, thin] (axis cs:4.06,0.316) -- (axis cs:4.30,0.22);
\draw[gray!55, thin] (axis cs:4.06,0.12) -- (axis cs:4.30,0.08);
\node[attackc, font=\scriptsize, anchor=west, inner sep=1.5pt] at (axis cs:4.30,0.50) {ins20};
\node[attackb, font=\scriptsize, anchor=west, inner sep=1.5pt] at (axis cs:4.30,0.36) {syn30};
\node[refline, font=\scriptsize, anchor=west, inner sep=1.5pt] at (axis cs:4.30,0.22) {theory};
\node[tango, font=\scriptsize, anchor=west, inner sep=1.5pt] at (axis cs:4.30,0.08) {del30};
\end{axis}
\end{tikzpicture}
\caption{Detection under edits falls as taps are added (rows $\{2\}$,
$\{1, 2\}$, $\{1, 2, 3\}$, and $\{1, 2, 3, 4\}$ of \Cref{tab:e3}). The
dashed curve scales the single-tap
del30 TPR by $(1-\rho)^{|\mathcal{D}| - 1}$ at $\rho = 0.3$. It shows only
the trend that \Cref{thm:robust} predicts, because the theorem concerns the
mean score under substitution.}
\label{fig:taps}
\end{figure}
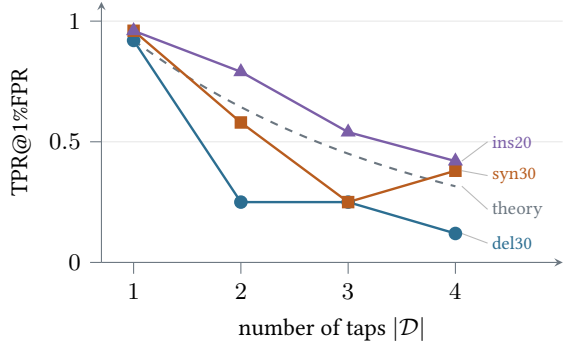

\FloatBarrier
\subsection{Enforcement side, bias, lag, and unmasking order}
\label{app:screen}

Either-side enforcement raises detection under deletion, and at $\beta =
6$ it costs little Qwen perplexity ($7.8$ against $7.4$). The other
variants of the selection experiment each cost more perplexity or leak
more to forgery (\Cref{tab:e9-screen}, rule in \Cref{app:protocol}).
Besides the fixed bias $\beta$, the experiment tried a \emph{target} bias,
which at each enforced position adds the smallest bias, up to a cap, that lifts the favored class
to a target probability. It also tried a lag-$1$ tap and \emph{order
steering}. Order steering targets positions whose pair cannot yet be
enforced, because neither neighbor at lag $\delta$ is unmasked. It halves
the confidence score that the sampler uses to rank these positions for
unmasking. The sampler then tends to unmask positions with an unmasked
neighbor first.

Enforcing from either side raises the enforced fraction from $0.61$ to
$0.74$ at $\beta = 8$ and from $0.62$ to $0.75$ at $\beta = 6$, and it
raises del30 TPR at both biases. A lag-$1$ tap reaches del30 TPR $0.94$,
against $0.77$ for lag $2$ at the same $\beta = 6$ and perplexity, but its
forgery success is $0.54$ against $0.38$. Order steering raises the
enforced fraction to about $0.9$, but across the three steered settings it
multiplies perplexity by $1.3$ to $1.7$ (Qwen and GPT-2), and at $\beta =
6$ it lowers AUROC to $0.974$.

\begin{table}[htbp]
\centering
\caption{Either-side enforcement at $\beta = 5$ (bold) has the lowest Qwen
perplexity among the eligible settings of the selection experiment. $48$
held-out C4 prompts, with their $48$ unwatermarked continuations as
negatives. \emph{Pair AUC / forge} is the attack of \Cref{tab:steal-pairs}
at $N = 48$. \emph{Target, cap} biases each enforced position just enough
to lift the favored class to the target probability, with a bias of at
most the cap. \emph{Steered} is order steering. $\circ$ marks eligible
settings, $\times$ marks eligible settings excluded for forgery, and
unmarked settings are ineligible (rule in \Cref{app:protocol}). Perplexity is over all continuations.}
\label{tab:e9-screen}
\footnotesize
\setlength{\tabcolsep}{3pt}
\resizebox{\linewidth}{!}{\begin{tabular}{@{}llrrrrrrrrr@{}}
\toprule
setting & & AUROC & clean & del30 & syn30 & ins20 & enforced fraction & PPL GPT-2 & PPL Qwen & pair AUC / forge \\
\midrule
one-sided, $\beta=8$ (reference) & & 0.999 & 0.98 & 0.79 & 0.98 & 0.98 & 0.61 & 24.0 & 9.7 & 0.61 / 0.36 \\
one-sided, $\beta=6$ & & 0.999 & 0.98 & 0.75 & 0.79 & 0.96 & 0.62 & 19.6 & 7.4 & 0.61 / 0.26 \\
either side, $\beta=8$ & $\times$ & 1.000 & 1.00 & 0.94 & 1.00 & 1.00 & 0.74 & 31.1 & 12.7 & 0.62 / 0.47 \\
either side, $\beta=6$ & & 1.000 & 1.00 & 0.77 & 0.96 & 0.96 & 0.75 & 19.6 & 7.8 & 0.63 / 0.38 \\
\textbf{either side, $\boldsymbol{\beta=5}$} & $\circ$ & \textbf{1.000} & \textbf{1.00} & \textbf{0.79} & \textbf{0.92} & \textbf{0.98} & \textbf{0.76} & \textbf{17.8} & \textbf{7.0} & \textbf{0.62 / 0.40} \\
either side, $\beta=4$ & & 0.998 & 0.96 & 0.73 & 0.81 & 0.81 & 0.77 & 16.0 & 5.9 & 0.61 / 0.23 \\
either side, target $0.8$, cap $8$ & $\circ$ & 1.000 & 1.00 & 0.92 & 0.96 & 0.98 & 0.71 & 25.7 & 11.0 & 0.62 / 0.42 \\
either side, target $0.9$, cap $8$ & $\circ$ & 1.000 & 1.00 & 0.92 & 1.00 & 1.00 & 0.72 & 27.8 & 11.7 & 0.62 / 0.34 \\
either side, target $0.95$, cap $8$ & $\times$ & 1.000 & 1.00 & 0.88 & 1.00 & 1.00 & 0.73 & 29.0 & 11.9 & 0.63 / 0.49 \\
either side, target $0.9$, cap $6$ & $\circ$ & 1.000 & 1.00 & 0.79 & 0.83 & 0.92 & 0.73 & 19.8 & 8.0 & 0.62 / 0.28 \\
either side, $\beta=6$, lag $1$ & $\times$ & 1.000 & 1.00 & 0.94 & 0.98 & 0.98 & 0.73 & 19.6 & 7.6 & 0.65 / 0.54 \\
either side, $\beta=6$, steered & & 0.974 & 0.96 & 0.85 & 0.94 & 0.94 & 0.90 & 33.9 & 11.2 & 0.61 / 0.40 \\
either side, $\beta=8$, steered & & 0.979 & 0.98 & 0.96 & 0.98 & 0.98 & 0.88 & 44.9 & 21.2 & 0.62 / 0.62 \\
either side, target $0.9$, cap $8$, steered & $\times$ & 1.000 & 1.00 & 0.96 & 1.00 & 1.00 & 0.86 & 36.3 & 17.8 & 0.63 / 0.57 \\
\midrule
unwatermarked & & & & & & & & 12.1 & 4.2 & \\
\bottomrule
\end{tabular}}
\end{table}

\FloatBarrier
\subsection{Detection grows with text length}

\Cref{prop:ez} predicts that the mean score grows as $\sqrt{T}$ in the
number $T$ of scored pairs, with slope $\varepsilon / \sigma_0$ when the
margin $\varepsilon$ does not depend on length. \Cref{fig:ztheory} checks
the $\sqrt{T}$ shape on $200$ LLaDA texts. We truncate each $128$-token text
to its first $16$, $32$, $48$, $64$, $96$, or $128$ tokens and score each
prefix. Some texts end early, so the full texts have $124.9$ scored pairs on
average. The slope $0.48$ comes from the margin measured on the full texts
($\hat\varepsilon = 0.227$), so the line and the data agree at full length
by construction. The test is whether shorter prefixes fall on the same
line. At $94$ scored pairs the measured mean score matches the line ($4.67$
in both cases). For shorter prefixes it is lower, so the margin is smaller
near the start of a text than over the whole text. The line reaches $z =
4$ at $T \approx 69$ scored pairs, and the measured scores reach it at
about $T = 78$. The prediction therefore underestimates the length of text that
detection needs by about $9$ scored pairs.

\begin{figure}[htbp]
\centering
\begin{tikzpicture}
\begin{axis}[
  width=0.55\linewidth, height=0.38\linewidth,
  font=\footnotesize,
  xlabel={$\sqrt{T}$ ($T$ = scored pairs)},
  ylabel={mean detector score $z$},
  xmin=0, xmax=13, ymin=-1.4, ymax=6.4,
  xtick={0,4,8,12}, ytick={0,2,4,6},
  axis lines=left, axis line style={mut, thin},
  tick style={mut, thin}, tick align=outside,
  ymajorgrids, grid style={gray!18, thin},
  clip=false,
]
\addplot[refline, dashed, line width=0.8pt, domain=0:12.0, samples=2] {0.48*x};
\addplot[tango, line width=0.9pt, mark=*, mark size=2.1pt] coordinates {
  (3.74,0.72) (5.48,1.46) (6.78,2.45) (7.87,3.29) (9.70,4.67) (11.17,5.38)};
\addplot[mut, line width=0.7pt, mark=square*, mark size=1.7pt] coordinates {
  (3.74,-0.28) (5.48,-0.24) (6.78,-0.32) (7.87,-0.28) (9.70,-0.30) (11.17,-0.34)};
\node[refline, font=\scriptsize, rotate=27, anchor=south west, inner sep=1pt]
  at (axis cs:2.6,1.55) {theory: slope $0.48$};
\node[tango, font=\scriptsize\bfseries, anchor=north east, inner sep=2pt] at (axis cs:9.4,5.9) {TANGO};
\node[mut, font=\scriptsize, anchor=north west, inner sep=2pt]
  at (axis cs:3.5,-0.5) {unwatermarked};
\end{axis}
\end{tikzpicture}
\caption{The mean score matches the predicted line at $94$ scored pairs
and falls below it for shorter prefixes. Unwatermarked text stays near
zero. Prefixes of $200$ LLaDA texts,
one-sided enforcement at $\beta = 6$, whitened coloring. The dashed line
has slope $\hat\varepsilon / \sigma_0 = 0.48$, with $\hat\varepsilon$
measured on the full texts.}
\label{fig:ztheory}
\end{figure}
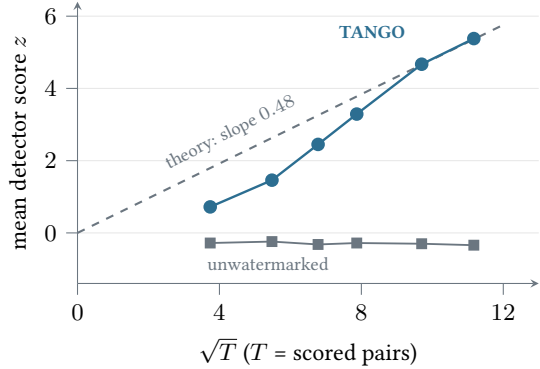

\FloatBarrier
\section{Additional related work}
\label{app:related}

This appendix extends \Cref{sec:related} with further watermarks for
diffusion language models, semantic watermarks, and attacks and defenses
around forgery.

\paragraph{Further watermarks for diffusion language models.}
DMark \citep{dmark2025} extends the context-hashed green-list bias of
\citet{kirchenbauer2023watermark} with predicted and bidirectional context.
\citet{bagchi2025ddlm} apply Gumbel-max sampling at every diffusion step,
seeded by position, and SAC-Copula \citep{li2026saccopula} replaces its
independent perturbations with locally correlated Gumbel fields.
\citet{zhao2026sketch} controls a global sketch of the whole sequence,
which gives an order-agnostic statistic. DenMark \citep{ma2026denmark},
concurrent with our work, embeds a semantic watermark in the denoising
process.

\paragraph{Semantic watermarks.}
Watermarks keyed by hashes of tokens break under paraphrase, which
motivated watermarks keyed by meaning. SIR derives the watermark logits
from an embedding of the preceding context \citep{liu2024sir}, SemStamp partitions a
sentence-embedding space and rejection-samples into keyed regions
\citep{hou2024semstamp}, and SemaMark discretizes context semantics
\citep{ren2024semamark}. These methods use meaning to make the key
assignment of a single token or sentence robust to paraphrase. TANGO uses meaning to make a pair
statistic survive synonym substitution. Like most semantic watermarks, it uses embeddings, but only to build the coloring once per key, so detection needs only the stored coloring. Image watermarks for diffusion
models share both ideas, embedding the mark inside the sampling process and
keying it to semantics. Hidden-in-the-Noise places the mark in the initial
diffusion noise \citep{arabi2024hidden}, and SEAL keys it to image
semantics \citep{arabi2025seal}.

\paragraph{Stealing, spoofing, and signatures.}
Watermark stealing is practical. With a modest number of texts an attacker
approximates the green list well enough to spoof or remove the watermark
\citep{jovanovic2024stealing}, and random-selection probing recovers the green lists of $n$-gram watermarks \citep{demark2024}. Comparing
output token frequencies also reveals token colors \citep{wu2024scts}, and
\citet{zhang2024mip} recover the green list of a single-key scheme with
mixed integer programming and use it to remove the watermark. A student model
can also learn to generate watermarked text by distillation, which lets an
attacker spoof a watermark without reading its key
\citep{gu2024learnability}. Spoofing was an early objection to
detector-based provenance \citep{sadasivan2023canai}, and
\citet{zhang2024sand} show that strong watermarking is impossible against
an attacker with a quality oracle and a perturbation oracle. Defenses
include statistical tests that flag learning-based spoofing
\citep{gloaguen2024spoofing}, and \citet{zhao2024sok} list robust,
unforgeable public attribution among the open problems of watermarking.
\citet{liu2024upv} use separate generation and detection networks so that
the public detector is hard to forge from, and Bileve embeds signature bits
to detect spoofing \citep{zhou2024bileve}. The cryptographic response embeds
a digital signature in the text \citep{fairoze2025publicly}, and
\citet{robustsig2026} make the signature robust to token substitutions. A
signature's guarantee is computational and holds against any
computationally bounded attacker. TANGO's guarantee is statistical. Under an idealized sampler, equal class masses, and an evenly varying favored class, expected token frequencies carry no information
about the key (\Cref{thm:unforgeable}), and the guarantee says nothing
about token pairs. Signature schemes embed the
signature in high-entropy blocks of text, for example by rejection sampling. TANGO adds one logit update over the vocabulary per biased position per denoising step, and its expected score decreases at a known rate under
random substitutions (\Cref{thm:robust}).

\FloatBarrier
\section{Qualitative samples}
\label{app:samples}

\Cref{fig:samples} shows continuations of three C4 prompts by the
unwatermarked model, TANGO, the red--green list, and the Gumbel rule at the
settings of \Cref{tab:headline}. Each entry is one sentence copied
verbatim from the full continuation. In these samples all four methods
produce readable text, and TANGO's text shows no visible trace of the
checksum.

\clearpage

\begin{figure}[htbp]
\centering
\qcard{
  \qprompt{\dots\ enjoying the sounds of holiday music, or taking a trip back in time to celebrate the traditions}
  \qrow{ink}{Unwatermarked}{Each experience enriches the holiday season, making it a time of personalization and cherished memories.}
  \qrow{tango}{TANGO}{Exchanging a wish list with Santa brings a sense of anticipation and joy, while the sounds of holiday music can transport you straight to the festive season.}
  \qrow{boost}{Red--green list}{It's a time to reunite with loved ones, create cherished memories, and appreciate the small joys that make the holiday season truly special.}
  \qrow{mut}{Gumbel}{Sharing your wish list with Santa can spark a sense of hope and anticipation, while the enchanting sounds of holiday music can transport you to cozy fireplaces and starlit nights.}
}
\qcard{
  \qprompt{Students have long applied to colleges and universities with applications that are heavy on test scores and grades \dots\ the founders of}
  \qrow{ink}{Unwatermarked}{However, over time, standardized test scores like the SAT and ACT (or A Levels in the UK) and academic grades have become significant components of the admissions process.}
  \qrow{tango}{TANGO}{Many founders recognize that reliance on test scores and grades does not always reflect a student's overall potential, character, or diverse interests.}
  \qrow{boost}{Red--green list}{While that's not necessarily wrong, the founders of colleges often emphasized the importance of considering a student's overall character, extracurricular activities, personal statements, letters of recommendation, and other elements that demonstrate well-roundedness.}
  \qrow{mut}{Gumbel}{Historically, founders wanted to assess a student's potential for academic success, extracurricular activities, personal qualities, and other factors.}
}
\qcard{
  \qprompt{Cluster comprises IBM's Opteron-based eServer 325 server and systems management software and storage devices that can run Linux and Windows operating systems}
  \qrow{ink}{Unwatermarked}{This setup is ideal for heterogeneous computing environments where multiple operating systems need to coexist and interact seamlessly.}
  \qrow{tango}{TANGO}{IBM's eServer 325 servers are known for their efficient performance and reliability, making them suitable for various computational needs.}
  \qrow{boost}{Red--green list}{This configuration allows for a high performance, scalable, and compatible server architecture that can accommodate various operating systems and applications.}
  \qrow{mut}{Gumbel}{This setup allows for flexibility and scalability, as it can run multiple operating systems concurrently, which can be beneficial for applications that require heterogeneous environments.}
}
\caption{Continuations of three C4 prompts at the settings of
\Cref{tab:headline}, each a verbatim span of the full continuation.}
\label{fig:samples}
\end{figure}

\end{document}